\documentclass[11pt,twoside]{article}
\usepackage{fullpage,amsmath,amsthm,amssymb,amsfonts}

\usepackage{times}

\usepackage{hyperref,xspace}

\usepackage{bm}

\usepackage[round,comma]{natbib}
\usepackage{amssymb}
\usepackage{amsmath}

\newcommand{\ie}{i.e.\ }
\newcommand{\eg}{e.g.\ }

\newcommand{\fr}[1]{\frac{1}{#1}}

\newcommand{\A}{{\mathcal A}}

\newcommand{\C}{{\mathcal C}}
\newcommand{\D}{{\mathcal D}}
\newcommand{\F}{{\mathcal F}}

\newcommand{\cM}{{\mathcal M}}

\newcommand{\cS}{\mathcal{S}}
\newcommand{\U}{\mathcal{U}}

\newcommand{\R}{{\mathbb R}}
\newcommand{\E}{{\mathbf E}}
\newcommand{\Var}{\mathbf{Var}}

\newcommand{\cond}{\ |\ }
\newcommand{\zo}{\{0,1\}}
\newcommand{\zon}{\{0,1\}^n}

\newcommand{\poly}{\mathrm{poly}}

\newcommand{\ti} \tilde

\newcommand{\eps}{\epsilon}

\newcommand{\logd}{\log{(1/\delta)}}

\newcommand{\rank}{\mbox{\tt{rank}}}

\newcommand{\equ}[1]{

\begin{equation}
#1
\end{equation}}
\newcommand{\equn}[1]{
$$
#1
$$}

\newcommand{\ignore}[1]{\relax}
\newcommand{\alequ}[1]{\begin{align} #1 \end{align}}
\newcommand{\alequn}[1]{\begin{align*} #1 \end{align*}}

\newcommand{\eat}[1]{}

\newtheorem{theorem}{Theorem}[section]
\newtheorem{lemma}[theorem]{Lemma}
\newtheorem{fact}[theorem]{Fact}
\newtheorem{corollary}[theorem]{Corollary}

\newtheorem{definition}[theorem]{Definition}

\newtheorem{problem}[theorem]{Problem}
\newtheorem{remark}[theorem]{Remark}

\author{Vitaly Feldman \\
IBM Research - Almaden \and Pravesh Kothari \\
University of Texas, Austin\thanks{Work done while the author was at IBM Research - Almaden.} \and Jan Vondr\'{a}k \\
IBM Research - Almaden\\
}

\eat{
\coltauthor{\Name{Vitaly Feldman} \Email{vitaly@post.harvard.edu} \\
\addr IBM Almaden, San Jose, CA
\AND
\Name{Pravesh Kothari} \Email{kothari@cs.utexas.edu} \\
\addr University of Texas, Austin, TX\thanks{Work done while the author was at IBM Research - Almaden.}
\AND
\Name{Jan Vondr\'ak} \Email{jvondrak@us.ibm.com} \\
\addr IBM Almaden, San Jose, CA
}
}

\newcommand{\WP}{\mbox{\tt{WP}}}
\newcommand{\DT}{\mbox{\tt{DT}}}
\newcommand{\AEFT}{\mbox{\tt{AEFT}}}
\newcommand{\degree}{\mbox{\tt{degree}}}
\newcommand{\xx}{\mathbf{x}}

\newcommand{\on}{\{-1,1\}}

\newcommand{\RR}{\mathbb R}

\newcommand{\N}{\mathbb{N}}

\newcommand{\mintitle}[1]{{\noindent {\bf #1}}}

\title{Representation, Approximation and Learning of Submodular Functions Using Low-rank Decision Trees}

\begin{document}

\maketitle

\begin{abstract}
We study the complexity of approximate representation and learning of submodular functions over the uniform distribution on the Boolean hypercube $\zon$. Our main result is the following structural theorem: any submodular function is $\eps$-close in $\ell_2$ to a real-valued decision tree (DT) of depth $O(1/\eps^2)$. This immediately implies that any submodular function is  $\eps$-close to a function of at most $2^{O(1/\eps^2)}$ variables and has a spectral $\ell_1$ norm of $2^{O(1/\eps^2)}$. It also implies the closest previous result that states that submodular functions can be approximated by polynomials of degree $O(1/\eps^2)$ \citep{CheraghchiKKL:12}.
 Our result is proved by constructing an approximation of a submodular function by a DT of rank $4/\eps^2$ and a proof that any rank-$r$ DT can be $\eps$-approximated by a DT of depth $\frac{5}{2}(r+\log(1/\eps))$.

We show that these structural results can be exploited to give an attribute-efficient PAC learning algorithm for submodular functions running in time $\tilde{O}(n^2) \cdot 2^{O(1/\eps^{4})}$. The best previous algorithm for the problem requires $n^{O(1/\eps^{2})}$ time and examples \citep{CheraghchiKKL:12} but works also in the agnostic setting. In addition, we give improved learning algorithms for a number of related settings.

We also prove that our PAC and agnostic learning algorithms are essentially optimal via two lower bounds: (1) an information-theoretic lower bound of $2^{\Omega(1/\eps^{2/3})}$ on the complexity of learning monotone submodular functions  in any reasonable model (including learning with value queries); (2) computational lower bound of $n^{\Omega(1/\eps^{2/3})}$ based on a reduction to learning of sparse parities with noise, widely-believed to be intractable. These are the first lower bounds for learning of submodular functions over the uniform distribution.
\end{abstract}

\eat{
\begin{keywords}
submodular function, decision tree, learning, uniform distribution
\end{keywords}
}

\section{Introduction} \label{sec:intro}

We study the problem of learning submodular functions and their (approximate) representation.
 Submodularity, a discrete analog of convexity, has played an essential role in combinatorial optimization~\citep{L83}.
It appears in many important settings including cuts in graphs~\citep{GW95,Q95,FFI00},
rank function of matroids~\citep{E70,F97}, set covering problems~\citep{F98}, and plant location problems~\citep{CFN77}.
Recently, interest in submodular functions has been revived by new applications
in algorithmic game theory as well as machine learning.
In machine learning, several applications \citep{GKS05,KGGK06,KSG08,KG11} have relied on the fact
that the information provided by a collection of sensors is a submodular function.
In algorithmic game theory, submodular functions have found application as {\em valuation functions}
with the property of diminishing returns \citep{LLN06,DNS05,Vondrak08,PSS08,DRY11}.

Wide-spread applications of submodular functions have recently inspired the question of whether and how such functions can be learned from random examples (of an unknown submodular function). The question was first formally considered by \citet{BalcanHarvey:12full} who motivate it by learning of valuations functions. Previously, reconstruction of such functions up to some multiplicative factor from value queries (which allow the learner to ask for the value of the function at any point) was also considered by \citet{GHIM09}. These works have lead to significant attention to several variants of the problem of learning submodular functions \citep{GuptaHRU:11,CheraghchiKKL:12,BadanidiyuruDFKNR:12,BalcanCIW:12,RaskhodnikovaYaroslavtsev:13soda}. We survey the prior work in more detail in Sections \ref{sec:our-results} and \ref{sec:prior-work}.

In this work we consider the setting in which the learner gets random and uniform examples of an unknown submodular function $f$ and its goal is to find a hypothesis function $h$ which $\eps$-approximates $f$ for a given $\eps > 0$. The main measures of the approximation error we use are the standard absolute error or $\ell_1$-distance, which equals $\E_{x\sim D}[|f(x)-h(x)|]$ and $\ell_2$-distance which equals $\sqrt{\E_{x\sim D}[(f(x)-h(x))^2]}$ (and upper-bounds the $\ell_1$ norm). This is essentially the PAC model \citep{Valiant:84} of learning applied to real-valued functions (as done for example by \citet{Haussler:92} and \citet{KearnsSS:94}). It is also closely related to learning of probabilistic concepts (which are concepts expressing the probability of the function being 1) in which the goal is to approximate the unknown probabilistic concept in $\ell_1$ \citep{KearnsSchapire:94}. As follows from the previous work \citep{BalcanHarvey:12full}, without assumptions on the distribution, learning a submodular function to a constant $\ell_1$ error requires an exponential number of random examples.
We therefore consider the problem with the distribution restricted to be uniform, a setting widely-studied in the context of learning Boolean functions in the PAC model (\eg \citet{LinialMN:93,ODonnellServedio:07}). This special case is also the focus of several other recent works on learning submodular functions \citep{GuptaHRU:11,CheraghchiKKL:12,RaskhodnikovaYaroslavtsev:13soda}.

\subsection{Our Results}
\label{sec:our-results}
We give three types of results on the problem of learning and approximating submodular function over the uniform distribution. First we show that submodular functions can be approximated by decision trees of low-rank. Then we show how such approximation can be exploited for learning. Finally, we show that our learning results are close to the best possible.

\smallskip
\noindent{\bf{Structural results:}} Our two key structural results can be summarized as follows. The first one shows that every submodular function can be approximated by a decision tree of low {\em rank}. The {\em rank} of a decision tree is a classic measure of complexity of decisions trees introduced by \citet{EhrenfeuchtHaussler:89}.
One way to define the rank of a decision tree $T$ (denoted by $\rank(T)$) is as the depth of the largest complete binary tree that can be embedded in $T$ (see Section~\ref{sec:prelims} for formal definitions).

\begin{theorem}
\label{th:submod-lip-rank-const-intro}
Let $f:\zon \rightarrow [0,1]$ be a submodular function and $\eps > 0$. There exists a real-valued binary decision tree $T$ of rank at most $4/\eps^2$ that approximates $f$ within $\ell_2$-error $\eps$.
\end{theorem}

This result is based on a decomposition technique of \citet{GuptaHRU:11} that shows that a submodular function $f$ can be decomposed into disjoint regions where $f$ is also $\alpha$-Lipschitz (for some $\alpha > 0$). We prove that this decomposition can be computed by a binary decision tree of rank $2/\alpha$. Our second result is that over the uniform distribution
a decision tree of rank $r$ can be $\eps$-approximated by a decision tree of depth $O(r+\log(1/\eps))$.

\begin{theorem}
\label{th:pruning-intro}
Let $T$ be a binary decision tree of rank $r$. Then for any integer $d \geq 0$,
$T$ truncated at depth $d = \frac{5}{2}(r+\log(1/\eps))$ gives a decision tree $T_{\leq d}$ such that, $\Pr_\U[T(x) \neq T_{\leq d}(x)] \leq \epsilon$.
\end{theorem}

It is well-known (\eg \citep{KushilevitzMansour:93}), that a decision tree of size $s$ (\ie with $s$ leaves) is $\eps$-close to the same decision tree pruned at depth $\log(s/\eps)$. It is also well-known that for any decision tree of size $s$ has rank of at most $\log s$. Therefore Theorem \ref{th:pruning-intro} (strictly) generalizes the size-based pruning. Another implication of this result is that several known algorithms for learning polynomial-size DTs over the uniform distribution (\eg \citep{KushilevitzMansour:93,GopalanKK:08}) can be easily shown to also learn DTs of logarithmic rank (which might have superpolynomial size). 

Combining Theorems \ref{th:submod-lip-rank-const-intro} and \ref{th:pruning-intro} we obtain that submodular functions can be approximated by shallow decision trees and consequently as functions depending on at most $2^{\poly(1/\epsilon)}$ variables.
\begin{corollary}
\label{cor:submod-tree-approx-intro}
Let $f:\zon \rightarrow [0,1]$ be a submodular function and $\eps > 0$. There exists a binary decision tree $T$ of depth $d  = O(1/\eps^2)$ with constants in the leaves such that $\|T - f\|_2 \leq \eps$. In particular, $T$ depends on at most $2^{O(1/\eps^2)}$ variables.
\end{corollary}
We remark that it is well-known that a DT of depth $d$ can be written as a polynomial of degree $d$. This gives a simple combinatorial proof of the low-degree approximation of \citep{CheraghchiKKL:12} which is based on an analysis of the noise stability of submodular functions. In addition, in our case the polynomial depends only on $2^{O(1/\eps^2)}$ variables, which is not true for the approximating polynomial constructed in \citep{CheraghchiKKL:12}.

\smallskip
\noindent{\bf{Algorithmic applications:}}
We show that these structural results can be used to obtain a number of new learning algorithms for submodular functions. One of the key issues in applying our approximation by a function of few variables is detecting the $2^{O(1/\eps^2)}$ variables that would suffice for approximating a submodular function given random examples alone. While for general functions this probably would not be an efficiently solvable problem, we show that a combination of (1) approximation of submodular functions by low-degree polynomials of low spectral (Fourier) $\ell_1$ norm  (implied by the DT approximation) and (2) the discrete concavity of submodular functions allow finding the necessary variables by looking at Fourier coefficients of degree at most 2.
\begin{lemma}
\label{lem:find-inf-vars-intro}
There exists an algorithm that given uniform random examples of values of a submodular function $f: \zon \rightarrow [0,1]$, finds a set of $2^{O(1/\eps^2)}$ variables $J$ such that there is a function $f_J$ depending only on the variables in $J$ and satisfying $\|f - f_J\|_2 \leq \eps$. The algorithm runs in time $n^2 \log (n) \cdot 2^{O(1/\eps^2)}$ and uses $\log (n) \cdot 2^{O(1/\eps^2)}$ random examples.
\end{lemma}

Combining this lemma with Corollary~\ref{cor:submod-tree-approx-intro} and using standard Fourier-based learning techniques, we obtain the following learning result in the PAC model.
\begin{theorem}
\label{th:pac-learn-submod-l1-intro}
There is an algorithm that given uniform random examples of any submodular function $f: \zon \rightarrow [0,1]$, outputs a function $h$, such that $\|f-h\|_2 \leq \epsilon$. The algorithm runs in time $\tilde{O}(n^2) \cdot 2^{O(1/\eps^4)}$ and uses $2^{O(1/\eps^{4})} \log n$ examples.
\end{theorem}
In the language of approximation algorithms, we give the first {\em efficient polynomial-time approximation scheme} (EPTAS) algorithms for the problem. We note that the best previously known algorithm for learning of submodular functions within $\ell_1$-error $\epsilon$ runs in time $n^{O(1/\epsilon^2)}$ \citep{CheraghchiKKL:12}, in other words is a PTAS (this algorithm works also in the agnostic setting).

\eat{
The hypothesis returned by our algorithm is a function of $2^{O(1/\eps^2)}$ variables only. This allows us to obtain a submodular function as a hypothesis (referred to as {\em proper} learning) in time $\tilde{O}(n^2) \cdot 2^{O(1/\eps^{4})} + 2^{2^{O(1/\eps^2)}}$ and using $2^{O(1/\eps^{4})} \log n$ random examples. Using a known connection between proper learning and property testing, we can also use this approach to {\em test} whether a given function is submodular or $\epsilon$-far from submodular (in the $\ell_1$ norm). We defer the details to the full version.}

We also give a faster algorithm for agnostic learning of submodular functions, provided that we have access to value queries (returning $f(x)$ for a given point $x \in \zon$).
\begin{theorem}
\label{th:agn-learn-valueq-intro}
Let $\C_s$ denote the class of all submodular functions from $\zon$ to $[0,1]$.
There is an agnostic learning algorithm that given access to value queries for a function $f:\zon \rightarrow [0,1]$, outputs a function $h$ such that $\|f-h\|_2 \leq \Delta + \epsilon$, where $\Delta = \min_{g \in \C_s}\{\|f-g\|_2\}$. The algorithm runs in time $\poly(n, 2^{1/\eps^2})$ and uses $\poly(\log n, 2^{1/\eps^2})$ value queries.
\end{theorem}
This algorithm is based on an attribute-efficient version of the Kushilevitz-Mansour algorithm \citep{KushilevitzMansour:93} for finding significant Fourier coefficients by \citet{Feldman:07jmlr}. We also show a different algorithm with the same agnostic guarantee but relative to the $\ell_1$-distance (and hence incomparable). In this case the algorithm is based on an attribute-efficient agnostic learning of decision trees which results from agnostic boosting \citep{KalaiKanade:09,Feldman:10ab} applied to the attribute-efficient algorithm for learning parities \citep{Feldman:07jmlr}.

Finally, we discuss the special case of submodular function with a discrete range $\{0,1,\ldots,k\}$ studied in a recent work of \citet{RaskhodnikovaYaroslavtsev:13soda}. We show that an adaptation of our techniques implies that such submodular functions can be {\em exactly} represented by rank-$2k$ decision trees.
This directly leads to new structural results and faster learning algorithms in this setting. A more detailed discussion appears in Section \ref{sec:pseudo-boolean}.

\eat{
\begin{theorem} (informal)
\label{th:agn-learn-valueq-kval-intro}
Let $\C_s^k$ denote the class of all submodular functions from $\zon$ to $\{0,1,\ldots,k\}$. There exists an agnostic learning algorithm that given access to value queries for a function $f: \zon \rightarrow \{0,1,\ldots,k\}$, outputs a function $h:\zon \rightarrow \{0,1,\ldots,k\}$ such that $\E_\U[|f - h|] \leq \Delta + \epsilon$, where $\Delta = \min_{g\in \C_s^k}\{\E_\U[|f - g|]\}$. The algorithm runs in time $\poly(n,2^k,1/\eps)$ and uses $\poly(\log{n},2^k,1/\eps)$ value queries.
\end{theorem}
}

\smallskip
\noindent{\bf{Lower bounds:}}
We prove that an exponential dependence on $\eps$ is necessary for learning of submodular functions (even monotone ones), in other words, there exists no fully polynomial-time approximation scheme (FPTAS) for the problem.

\begin{theorem}
\label{th:PAC-hardness-intro}
PAC-learning monotone submodular functions with range $[0,1]$ within $\ell_1$-error of $\eps>0$ requires $2^{\Omega(\eps^{-2/3})}$ value queries to $f$.
\end{theorem}

Our proof shows that any function $g$ of $t$ variables can be embedded into a submodular function $f_g$ over $2t$ variables in a way that any approximation of $f_g$ to accuracy $\theta(t^{-3/2})$ would yield a $1/4$ approximation of $g$. The latter is well known to require $\Omega(2^{t})$ random examples (or even value queries). This result implies optimality (up to the constant in the power of $\eps$) of our PAC learning algorithms for submodular functions.

Further, we prove that agnostic learning of monotone submodular functions is computationally hard via a reduction from learning sparse parities with noise.

\begin{theorem}
\label{th:agnostic-hardness-intro}
Agnostic learning of monotone submodular functions with range $[0,1]$ within $\ell_1$-error of $\eps>0$ in time $T(n,1/\eps)$ would imply learning of parities of size $\eps^{-2/3}$ with noise of rate $\eta$ in time $\poly(n,\frac{1}{\eps(1-2\eta)}) + 2T(n, \frac{c}{\eps(1-2\eta)})$ for some fixed constant $c$.
\end{theorem}

Learning of sparse parities with noise is a well-studied open problem in learning theory closely related to problems in coding theory and cryptography. It is known to be at least as hard as learning of DNF expression and juntas over the uniform distribution \citep{FGKP:09}. The trivial algorithm for learning parities on $k$ variables from random examples corrupted by random noise of rate $\eta$ takes time $n^k \cdot \poly(\frac{1}{1-2\eta})$. The only known improvement to this is an elegant algorithm of \citet{ValiantG:12} which runs in time $n^{0.8k} \cdot \poly(\frac{1}{1-2\eta})$.

These results suggest that agnostic learning of monotone submodular functions in time $n^{o(\eps^{-2/3})}$ would require a breakthrough in our understanding of these long-standing open problems. In particular, a running time such as $2^{\poly(1/\epsilon)} \poly(n)$,
which we achieve in the PAC model, cannot be achieved for agnostic learning of submodular functions. In other words, we show that the agnostic learning algorithm of \citet{CheraghchiKKL:12} is likely close to optimal. We note that this lower bound does not hold for boolean submodular functions. Monotone boolean submodular functions are disjunctions and hence are agnostically learnable in $n^{O(\log(1/\eps))}$ time. For further details on lower bounds we refer the reader to Section \ref{sec:lower-bounds}.


\subsection{Related Work}
\label{sec:prior-work}
Below we briefly mention some of the other related work. We direct the reader to \citep{BalcanHarvey:12full} for a detailed survey. Balcan and Harvey study learning of submodular functions without assumptions on the distribution and also require that the algorithm output a value which is within a multiplicative approximation factor of the true value with probability $\geq 1 - \eps$ (the model is referred to as {\em PMAC learning}). This is a very demanding setting and indeed one of the main results in \citep{BalcanHarvey:12full} is a factor-$\sqrt[3]{n}$ inapproximability bound for submodular functions. This notion of approximation is also considered in subsequent works \citep{BadanidiyuruDFKNR:12,BalcanCIW:12} where upper and lower approximation bounds are given for other related classes of functions such as XOS and subadditive. The lower bound of \citet{BalcanHarvey:12full} also implies hardness of learning of submodular function with $\ell_1$ (or $\ell_2$) error: it is impossible to learn a submodular function $f:\{0,1\}^n \rightarrow [0,1]$ in $\poly(n)$ time within any nontrivial $\ell_1$ error over general distributions.
We emphasize that these strong lower bounds rely on a very specific distribution concentrated on a sparse set of points, and show that this setting is very different from the setting of uniform/product distributions which is the focus of this paper.

For product distributions, Balcan and Harvey show that 1-Lipschitz submodular functions of minimum nonzero value at least $1$ have concentration properties implying a PMAC algorithm providing an $O(\log \frac{1}{\eps})$-factor approximation except for an $\epsilon$-fraction of points, using $O(\frac{1}{\epsilon} n \log n)$ samples \citep{BalcanHarvey:12full}. In our setting,
we have no assumption on the minimum nonzero value, and we are interested in the additive $\ell_1$-error rather than multiplicative approximation.

\citet{GuptaHRU:11} show that submodular functions can be $\epsilon$-approximated by a collection of $n^{O(1/\epsilon^2)}$ $\epsilon^2$-Lipschitz submodular functions. Each $\epsilon^2$-Lipschitz submodular function can be $\epsilon$-approximated by a constant. This leads to a learning algorithm running in time $n^{O(1/\epsilon^2)}$, which however requires value oracle access to the target function, in order to build the collection. Their decomposition is also the basis of our approach. We remark that our algorithm can be directly translated into a faster algorithm for the private data release which motivated the problem in \citep{GuptaHRU:11}. However, for one of their main examples which is privately releasing disjunctions one does not need the full generality of submodular functions. Coverage functions suffice and for those even faster algorithms are now known \citep{CheraghchiKKL:12,FeldmanKothari:13covman}.


In a concurrent work, \citet{FeldmanKothari:13covman} consider learning of coverage functions. Coverage functions are a simple subclass of submodular functions which can be characterized as non-negative combinations of monotone disjunctions. They show that over the uniform distribution any coverage function can be approximated by a polynomial of degree $\log(1/\eps)$ over $O(1/\eps^2)$ variables and also prove that coverage functions can be PAC learned in fully-polynomial time (that is, with polynomial dependence on both $n$ and $1/\eps$). Note that our lower bounds rule out the possibility of such algorithms for all submodular functions. Their techniques are different from ours (aside from applications of standard Fourier representation-based algorithms).

\section{Preliminaries} \label{sec:prelims}


We work with Boolean functions on $\zon$. Let $\U$ denote the uniform distribution over $\zon$.

\smallskip
\mintitle{Submodularity}
A set function $f:2^N \rightarrow \RR$ is submodular if $f(A \cup B) + f(A \cap B) \leq f(A) + f(B)$ for all $A,B \subseteq N$.
In this paper, we work with an equivalent description of set functions as functions on the hypercube $\zon$.

For $x \in \zon$, $b \in \zo$ and $i \in n$, let $x_{i\leftarrow b}$ denote the vector in $\zo^n$ that equals $x$  with $i$-th coordinate set to $b$. For a function $f:\zon \rightarrow \R$ and index $i \in [n]$ we define
$\partial_i f(x) = f(x_{i\leftarrow 1}) - f(x_{i\leftarrow 0})$. A function $f:\zon \rightarrow \RR$ is submodular iff $\partial_i f$ is a non-increasing function for each $i \in [n]$, or equivalently, for all $i \neq j$, $\partial_{i,j} f(x) = \partial_i (\partial_j f(x)) \leq 0$. A function $f:\zon \rightarrow \RR$ is $\alpha$-{\bf Lipschitz} if $\partial_i f(x) \in [-\alpha, \alpha]$ for all $i \in [n], x \in \zon$.

\smallskip
\mintitle{Absolute error vs.~ Error relative to norm:}
In our results, we typically assume that the values of $f(x)$ are in a bounded interval $[0,1]$,
and our goal is to learn $f$ with an additive error of $\epsilon$. Some prior work considered an error relative to the norm of $f$, for example at most $\epsilon \|f\|_1$ \citep{CheraghchiKKL:12}. In fact, it is known that for nonnegative submodular functions, $\|f\|_1 = \E[f] \geq \frac14 \|f\|_\infty$ and hence this does not make much difference. If we scale $f(x)$ by $1 / (4 \|f\|_1)$, we obtain a function with values in $[0,1]$. Learning this function within an additive error of $\epsilon$ is equivalent to learning the original function within an error of $4\epsilon \|f\|_1$.

\smallskip
\mintitle{Decision Trees:}
We use $\xx_1,\xx_2,\ldots,\xx_n$ to refer to $n$ functions on $\{0,1\}^n$ such that $\xx_i(x) = x_i$. Let $X = \{\xx_1,\xx_2,\ldots,\xx_n\}$. We represent real-valued functions over $\zon$ using binary decision trees in which each leaf can itself be any real-valued function. Specifically, a function is represented as binary tree $T$ in which each internal node labeled by some variable $\xx \in X$ and each leaf $\ell$ labeled by some real-valued function $f_\ell$ over variables not restricted on the path to the leaf. We refer to a decision tree in which each leaf is labeled by a function from some set of functions $\F$ as $\F$-valued. If $\F$ contains only constants from the domain of the function then we obtain the usual decision trees.

For a decision tree $T$ with variable $\xx_r \in X$ at the root we denote by $T_0$ ($T_1$) the left subtree of $T$ (the right subtree, respectively). The value of the tree on a point $x$ is computed in the standard way: if the tree is a leaf $\ell$ then $T(x) = f_\ell(x_{X[v]})$, where $X[v]$ is the set of indices of variables which are not restricted on the path to $\ell$ and $x_{X[v]}$ is the substring of $x$ containing all the coordinates in $X[v]$. If $T$ is not a leaf then $T(x)=T_{\xx_r(x)}(x)$ where $\xx_r$ is the variable at the root of $T$.

The {\em rank} of a decision tree $T$ is defined as follows \citep{EhrenfeuchtHaussler:89}. If $T$ is a leaf, then $\rank(T) = 0$. Otherwise:
\equn{
\rank(T)=\left\{
     \begin{array}{ll}
       \max\{\rank(T_0),\rank(T_1)\} & \hbox{if  }\rank(T_0) \neq \rank(T_1) ; \\
       \rank(T_0) + 1, & \hbox{otherwise.}
     \end{array}
   \right.
}

The {\em depth} of a node $v$ in a tree $T$ is the length of the path the root of $T$ to $v$. The depth of a tree is the depth of its deepest leaf. For any node $v \in T$ we denote by $T[v]$ the sub-tree rooted at that node. We also use $T$ to refer to the function computed by $T$.

\smallskip
\mintitle{Fourier Analysis on the Boolean Cube}
We define the notions of inner product and norms, which we consider with respect to $\U$.
For two functions $f,g: \zon \rightarrow \R$, the inner product of $f$ and $g$ is defined as $\langle f, g \rangle = \E_{x \sim \U} [f(x) \cdot g(x)]$. The $\ell_1$ and $\ell_2$ norms of $f$ are defined by $||f||_1 =  \E_{x \sim \U} [|f(x)|]$ and $||f||_2 =  (\E_{x \sim \U} [f(x)^2])^{1/2}$ respectively.

For $S \subseteq [n]$, the parity function $\chi_S:\zon \rightarrow \on$ is defined by
$ \chi_S(x) = (-1)^{\sum_{i \in S} x_i}$. The parities form an orthonormal basis for functions on $\zon$ under the inner product product with respect to the uniform distribution. Thus, every function $f: \zon \rightarrow \R$ can be written as a real linear combination of parities. The coefficients of the linear combination are referred to as Fourier coefficients of $f$.
For $f:\zon \rightarrow \R$ and $S \subseteq [n]$, the Fourier coefficient $\hat{f}(S)$ is given by $\hat{f}(S) = \langle f, \chi_S \rangle.$ For any Fourier coefficient $\hat{f}(S)$, $|S|$ is called the \emph{degree} of the coefficient.

The Fourier expansion of $f$ is given by $ f(x) = \sum_{S \subseteq [n]} \hat{f}(S) \chi_S(x).$ The degree of highest degree non-zero Fourier coefficient of $f$ is referred to as the {\em Fourier degree} of $f$. Note that Fourier degree of $f$ is exactly the polynomial degree of $f$ when viewed over $\on^n$ instead of $\zon$ and therefore it is also equal to the polynomial degree of $f$ over $\zon$. Let $f: \zon \rightarrow \R$ and $\hat{f}: 2^{[n]} \rightarrow \R$ be its Fourier Transform. The {\em spectral $\ell_1$ norm} of $f$ is defined as $$ ||\hat{f}||_1 = \sum_{S \subseteq [n]} |\hat{f}(S)|.$$

The Fourier transform of partial derivatives satisfies:
$\partial_i f(x) = 2 \sum_{S \ni i}\hat{f}(S)\chi_{S\setminus\{i\}}(x)$, and
$\partial_{i,j} f(x) = 4 \sum_{S \ni i,j}\hat{f}(S)\chi_{S\setminus\{i,j\}}(x)$.

\smallskip
\mintitle{Learning Models}
Our learning algorithms are in one of two standard models of learning. The first one assumes that the learner has access to random examples of an unknown function from a known set of functions. This model is a generalization of Valiant's PAC learning model to real-valued functions \citep{Valiant:84,Haussler:92}.
\begin{definition}[PAC $\ell_1$-learning]
Let $\F$ be a class of real-valued functions on $\zon$ and let $\D$ be a distribution on $\zon$. An algorithm $\A$ PAC learns $\F$ on $\D$, if for every $\epsilon > 0$ and any target function $f \in \F$, given access to random  independent samples from $\D$ labeled by $f$, with probability at least $\frac{2}{3}$,  $\A$ returns a hypothesis $h$ such that $\E_{x \sim \D} [ |f (x) - h(x) | ] \leq  \epsilon.$ $\A$ is said to be \em{proper} if $h \in \F$.
\end{definition}
While in general Valiant's model does not make assumptions on the distribution $\D$, here we only consider the {\em distribution-specific} version of the model in which the distribution is fixed and is uniform over $\zon$. The error parameter $\eps$ in the Boolean case measures probability of misclassification.
Agnostic learning generalizes the definition of PAC learning to scenarios where one cannot assume that the input labels are consistent with a function from a given class \citep{Haussler:92,KearnsSS:94} (for example as a result of noise in the labels).
\begin{definition}[Agnostic $\ell_1$-learning]
Let $\F$ be a class of real-valued functions from $\zon$ to $[0,1]$ and let $\D$ be any fixed distribution on $\zon$. For any function $f$, let $\mbox{opt}(f,\F)$ be defined as: $$\mbox{opt}(f,\F) =  \inf_{g \in \F} \E_{x \sim \D} [ |g(x) - f(x) |] .$$ An algorithm $\A$, is said to agnostically learn $\F$ on $\D$ if for every $\epsilon> 0$ and any function $f :\zon \rightarrow [0,1]$, given access to random independent examples of $f$ drawn from $\D$, with probability at least $\frac{2}{3}$, $\A$ outputs a hypothesis $h$ such that $$\E_{x \sim \D} [ |h(x)- f(x)| ] \leq \mbox{opt}(f,\F) + \epsilon.$$
\end{definition}
The $\ell_2$ versions of these models are defined analogously.

\section{Approximation of Submodular Functions by Low-Rank Decision Trees}
We now prove that any bounded submodular function can be represented as a low-rank decision tree with $\alpha$-Lipschitz submodular functions in the leaves. Our construction follows closely the construction of \citet{GuptaHRU:11}. They show that for every submodular $f$ there exists a decomposition of $\zon$ into $n^{O(1/\alpha)}$ disjoint regions restricted to each of which $f$ is $\alpha$-Lipschitz submodular. In essence, we give a binary decision tree representation of the decomposition from \citep{GuptaHRU:11} and then prove that the decision tree has rank $O(1/\alpha)$.

\begin{theorem}
\label{th:submod-lip-rank}
Let $f:\zon \rightarrow [0,1]$ be a submodular function and $\alpha > 0$. Let $\F_\alpha$ denote the set of all $\alpha$-Lipschitz submodular functions with range $[0,1]$ over at most $n$ Boolean variables. Then $f$ can be computed by an $\F_\alpha$-valued binary decision tree $T$ of rank $r \leq 2/\alpha$.
\end{theorem}
We first prove the claim that decomposes a submodular function $f$ into regions where $f$ where discrete derivatives of $f$ are upper-bounded by $\alpha$ everywhere: we call this property $\alpha$-monotone decreasing.

\begin{definition}
For $\alpha \in \R$, $f$ is $\alpha$-monotone decreasing if for all $i \in [n]$ and $x \in \zon$, $\partial_i f (x) \leq \alpha$.
\end{definition}

We remark that $\alpha$-Lipschitzness is equivalent to discrete derivatives being in the range $[-\alpha,\alpha]$, i.e. $f$ as well as $-f$ being $\alpha$-monotone decreasing.

\begin{lemma}
\label{lem:submod-lip-oneside}
For $\alpha > 0$ let $f:\zon \rightarrow [0,1]$ be a submodular function. Let $\cM_\alpha$ denote the set of all $\alpha$-monotone decreasing submodular functions with range $[0,1]$ over at most $n$ Boolean variables. $f$ can be computed by a $\cM_\alpha$-valued binary decision tree $T$ of rank $r \leq 1/\alpha$.
\end{lemma}
\begin{proof}
The tree $T$ is constructed recursively as follows: if $n=0$ then the function is a constant which can be computed by a single leaf. If $f$ is $\alpha$-monotone decreasing then $T$ is equal to the leaf computing $f$. Otherwise, if $f$ is not $\alpha$-monotone decreasing then there exists $i\in [n]$ and $z \in \zon$ such that $\partial_i f (z) > \alpha$. In fact, submodularity of $f$ implies that $\partial_i f$ is monotone decreasing and, in particular, $\partial_i f (\bar{0}) \geq \partial_i f (z) > \alpha$. We label the root with $\xx_i$ and build the trees $T_0$ and $T_1$ for $f$ restricted to points $x$ such that $x_i = 0$ and $x_i = 1$, respectively (viewed as a function over $\zo^{n-1}$). Note that both restrictions preserve submodularity and $\alpha$-monotonicity of $f$.

By definition, this binary tree computes $f(x)$ and its leaves are $\alpha$-monotone decreasing submodular functions. It remains to compute the rank of $T$. For any node $v \in T$, we let $X[v] \subseteq [n]$ be the set of indices of variables that are not set on the path to $v$, let $\bar{X}[v] = [n]\setminus X[v]$ and let $y[v] \in \zo^{\bar{X}[v]}$ denote the values of the variables that were set. Let $\zo^{X[v]}$ be the subcube of points in $\zo^n$ that reach $v$, namely points $x$ such that $x_{X[v]} = y[v]$. Let $f[v](x) = T[v](x)$ be the restriction of $f$ to the subcube. Note that the vector of all $0$'s, $\bar{0}$ in the $\zo^{X[v]}$ subcube corresponds to the point which equals $y[v]$ on coordinates in $\bar{X}[v]$ and $0$ on all other coordinates. We refer to this point as $x[v]$.

Let $M = \max_x\{f(x)\}$. We prove by induction on the depth of $T[v]$ that for any node $v \in T$,
\equ{\rank(T[v]) \leq \frac{M - f[v](\bar{0})}{\alpha} .\label{eq:value-add}}
This is obviously true if $v$ is a leaf. Now, let $v$ be an internal node $v$ with label $\xx_i$. Let $v_0$ and $v_1$ denote the roots of $T[v]_0$ and $T[v]_1$, respectively.
For $v_0$, $x[v_0] = x[v]$ and therefore $f[v](\bar{0}) = f[v_0](\bar{0})$. By inductive hypothesis, this implies that
\equ{\rank[T[v_0]]\leq \frac{M - f[v_0](\bar{0})}{\alpha} = \frac{M - f[v](\bar{0})}{\alpha}\ .\label{eq:rank-t0}}

We know that $\partial_i f[v] (\bar{0}) > \alpha$. By definition, $\partial_i f[v] (\bar{0}) = f[v] (\bar{0}_{i \leftarrow 1}) - f[v] (\bar{0})$. At the same time, $f[v] (\bar{0}_{i \leftarrow 1}) = f(x[v]_{i \leftarrow 1}) = f(x[v_1]) = f[v_1](\bar{0})$. Therefore, $f[v_1](\bar{0}) \geq f[v](\bar{0}) + \alpha$. By the inductive hypothesis, this implies that
\equ{\rank[T[v_1]] \leq \frac{M - f[v_1](\bar{0})}{\alpha} \leq \frac{M - f[v](\bar{0}) - \alpha}{\alpha} = \frac{M - f[v](\bar{0})}{\alpha} - 1\ .\label{eq:rank-t1}}
Combining equations (\ref{eq:rank-t0}) and (\ref{eq:rank-t1}) and using the definition of the rank we obtain that equation (\ref{eq:value-add}) holds for $v$.

The claim now follows since $f$ has range $[0,1]$ and thus $M \leq 1$ and $f(\bar{0}) \geq 0$.
\end{proof}
We note that for monotone functions Lemma \ref{lem:submod-lip-oneside} implies Theorem \ref{th:submod-lip-rank} since discrete derivatives of a monotone function are non-negative. As in the construction in \citep{GuptaHRU:11}, the extension to the non-monotone case is based on observing that for any submodular function $f$, the function $\bar{f}(x) = f(\neg x)$ is also submodular, where $\neg x$ is obtained from $x$ by flipping every bit.

\begin{proof}[Proof of Theorem \ref{th:submod-lip-rank}]
We first apply Lemma \ref{lem:submod-lip-oneside} to obtain an $\cM_\alpha$-valued decision tree $T'$ for $f$ of rank $\leq 1/\alpha$. Now let $\ell$ be any leaf of $T'$ and let $f[\ell]$ denote $f$ restricted to $\ell$. As before, let $X[\ell] \subseteq [n]$ be the set of indices of variables that are not restricted on the path to $\ell$ and let $\zo^{X[\ell]}$ be the subcube of points in $\zo^n$ that reach $\ell$. We now use Lemma \ref{lem:submod-lip-oneside} to obtain an $\cM_\alpha$-valued decision tree $T_\ell$ for $\overline{f[\ell]}$ of rank $\leq 1/\alpha$. We denote by $\neg T_\ell$ the tree computing the function $T_\ell(\neg z)$. It is obtained from $T_\ell$ by swapping the subtrees of each node and replacing each function $g(z)$ in a leaf with $g(\neg z)$. We replace each leaf $\ell$ of $T'$ by $\neg T_\ell$ and let $T$ be the resulting tree. To prove the theorem we  establish the following properties of $T$.
\begin{enumerate}
\item Correctness: we claim that $T(x)$ computes $f(x)$. To see this note that for each leaf $\ell$ of $T'$, $\neg T_\ell(z)$ computes $T_\ell(\neg z) = \overline{f[\ell]}(\neg z) = f[\ell](z)$. Hence $T(x) = T'(x) = f(x)$.
\item $\alpha$-Lipschitzness of leaves: by our assumption, $f[\ell]$ is an $\alpha$-monotone decreasing function over $\zo^{{X[\ell]}}$ and therefore $\partial_i f[\ell] (z) \geq -\alpha$ for all $i\in {X[\ell]}$ and $z\in \zo^{{X[\ell]}}$. This means that for all $i\in {X[\ell]}$ and $z\in \zo^{{X[\ell]}}$,  \equ{ \partial_i \overline{f[\ell]}(z) = -\partial_i f[\ell](\neg z) \leq \alpha \label{eq:lip-one-side}.} Further, let $\kappa$ be a leaf of $T_\ell$ computing a function $\overline{f[\ell]}[\kappa]$. By Lemma \ref{lem:submod-lip-oneside}, $\overline{f[\ell]}[\kappa]$ is $\alpha$-monotone decreasing. Together with equation \ref{eq:lip-one-side} this implies that $\overline{f[\ell]}[\kappa]$ is $\alpha$-Lipschitz. In $\neg T_\ell$, $\overline{f[\ell]}[\kappa](z)$ is replaced by $\overline{f[\ell]}[\kappa](\neg z)$. This operation preserves $\alpha$-Lipschitzness and therefore all leaves of $T$ are $\alpha$-Lipschitz functions.
\item Submodularity of the leaf functions: for each leaf $\ell$, $f[\ell]$ is submodular simply because it is a restriction of $f$ to a subcube.
\item Rank: by Lemma \ref{lem:submod-lip-oneside}, $\rank(T') \leq 2/\alpha$ and for every leaf $\ell$ of $T'$, $\rank(\neg T_\ell) = \rank(T_\ell) \leq 1/\alpha$. As can be easily seen from the definition of rank, replacing each leaf of $T'$ by a tree of rank at most $1/\alpha$ can increase the rank of the resulting tree by at most $1/\alpha$. Hence the rank of $T$ is at most $2/\alpha$.
\end{enumerate}
\end{proof}


\subsection{Approximation of Leaves}
An important property of the decision tree representation is that it decomposes a function into disjoint regions. This implies that approximating the function over the whole domain can be reduced to approximating the function over individual regions with the same error parameter. Then, as in \citep{GuptaHRU:11}, we can use concentration properties of $\alpha$-Lipschitz submodular functions on the uniform distribution $\U$ over $\zon$ to approximate each $\alpha$-Lipschitz submodular functions by a constant.

Formally we state the following lemma which allows the use of any loss function $L$.
\begin{lemma}
\label{lem:leaf-approx}
For a set of functions $\F$, let $T$ be an $\F$-valued binary decision tree, $D$ be any distribution over $\zon$ and $L:\R \times \R \rightarrow \R$ be any real-valued (loss) function. For each leaf $\ell \in T$, let $D[\ell]$ be the distribution over $\zo^{X[\ell]}$ that equals  $D$ conditioned on $x$ reaching $\ell$; let $g_\ell$ be a function that satisfies $$\E_{z\sim D[\ell]} \left[L\left(T[\ell](z),g_\ell(z)\right)\right] \leq \eps. $$ Let $T'$ be the tree obtained from $T$ by replacing each function in a leaf $\ell$ with the corresponding $g_\ell$.
Then $\E_{x\sim D} [L(T(x),T'(x))] \leq \eps$.
\end{lemma}
\begin{proof}
For a leaf $\ell \in T$, let $y[\ell] \in \zo^{\bar{X}[\ell]}$ denote the values of the variables that were set on the path to $\ell$. Note that the subcube $\zo^{X[\ell]}$ corresponds to the points $x \in \zon$ such that $x_{X[\ell]} = y[\ell]$.
\alequn{\E_{x\sim D} [L(T(x),T'(x))] &= \sum_{\ell \in T} \E_{x\sim D} \left[L(T(x),T'(x)) \cond x_{X[\ell]} = y[\ell]\right] \cdot \Pr_{x \sim D}\left[x_{X[\ell]} = y[\ell]\right] \\ &=
\sum_{\ell \in T} \E_{z\sim D[\ell]} \left[L(T[\ell](z),g_\ell(z))\right] \cdot \Pr_{x \sim D}\left[x_{X[\ell]} = y[\ell]\right] \\ & \leq \sum_{\ell \in T} \eps \cdot \Pr_{x \sim D}\left[x_{X[\ell]} = y[\ell]\right] = \eps \ . }
\end{proof}

It is known that $1$-Lipschitz submodular functions satisfy strong concentration properties over the uniform distribution $\U$ over $\zon$ \citep{BoucheronLM:00, Vondrak10, BalcanHarvey:12full}, with standard deviation $O(\sqrt{\E[f]})$ and exponentially decaying tails. For our purposes we do not need the exponential tail bounds and instead we state the following simple bound on variance. 
\begin{lemma}
\label{lem:Lipschitz}
For any $\alpha$-Lipschitz submodular function $f:\{0,1\}^n \rightarrow \R_+$,
$$ \Var_\U[f] \leq 2 \alpha \cdot \E_\U[f].$$
\end{lemma}
\begin{proof}
By the Efron-Stein inequality (see \citep{BoucheronLM:00}),
$$\Var_\U[f] \leq \fr{2}\sum_{i\in [n]}\E_\U[(\partial_i f)^2] \leq \fr{2}\max_{i\in [n]} \E_\U[|\partial_i f|] \cdot \sum_{i\in [n]}\E_\U[|\partial_i f|] \leq \alpha \cdot \fr{2}\sum_{i\in [n]}\E_\U[|\partial_i f|]\ .$$
We can now use the fact that non-negative submodular functions are $2$-self-bounding  \citep{Vondrak10}, and hence
$\sum_{i\in [n]} \E_\U[|\partial_i f|] = 2 \E_{x \sim \U}[ \sum_{i: f(x \oplus e_i) < f(x)} (f(x) - f(x \oplus e_i))] \leq 4 \E_\U[f]$.
\end{proof}

We can now finish the proof of Theorem \ref{th:submod-lip-rank-const-intro}.
\begin{proof}[Proof of Theorem~\ref{th:submod-lip-rank-const-intro}]
Let $T'$ be the $\F_\alpha$-valued decision tree for $f$ given by Theorem \ref{th:submod-lip-rank} with $\alpha = \eps^2/2$. For every leaf $\ell$ we replace the function $T'[\ell]$ at that leaf by the constant $\E_\U[T'[\ell]]$ (here the uniform distribution is over $\zo^{X[\ell]}$) and let $T$ be the resulting tree.

Cor.~\ref{lem:Lipschitz} implies that for any $\eps^2/2$-Lipschitz submodular function $g:\zo^m \rightarrow [0,1]$,
$\Var_\U[g] = \E_\U[(g-\E_\U[g])^2] \leq 2 \frac{\eps^2}{2} \E_\U[g] \leq \eps^2$. For every leaf $\ell \in T'$, $T'[\ell]$ is $\eps^2/2$-Lipschitz and hence, $$\E_\U[(T'[\ell](z) - T[\ell](z))^2] = \E_\U[(T'[\ell](z) - \E_\U[T'[\ell]])^2] \leq \eps^2\ .$$
By Lemma \ref{lem:leaf-approx} (with $L(a,b) = (a-b)^2$), we obtain that $\E_\U[(T(x) - f(x))^2] \leq \eps^2$.
\end{proof}

\section{Approximation of Low-Rank Decision Trees by Shallow Decision Trees}
We show that over any constant-bounded product distribution $D$, a decision tree of rank $r$ can be $\eps$-approximated by a decision tree of depth $O(r+\log(1/\eps))$. The approximating decision tree is simply the original tree pruned at depth $d=O(r+\log(1/\eps))$.

For a vector $\mu \in [0,1]^n$ we denote by $D_{\mu}$ the product distribution over $\zon$, such that $\Pr_{D_{\mu}}[x_i=1] = \mu_i$. For $\alpha \in [0,1/2]$ a product distribution $D_\mu$ is $\alpha$-bounded if $\mu \in [\alpha,1-\alpha]^n$.
For a decision tree $T$ and integer $d \geq 0$ we denote by $T^{\leq d}$ a decision tree in which all internal nodes at depth $d$ are replaced by a leaf computing constant $0$.
\begin{theorem}
\label{th:pruning}(Theorem~\ref{th:pruning-intro} restated)
For a set of functions $\F$ let $T$ be a $\F$-valued decision tree of rank $r$, and let $D_\mu$ be an $\alpha$-bounded product distribution for some $\alpha \in (0,1/2]$. Then for any integer $d \geq 0$, $$\Pr_{D_\mu}[T^{\leq d}(x) \neq T(x)] \leq 2^{r-1} \cdot \left(1-\frac{\alpha}{2}\right)^{d}\ . $$ In particular, for $d = \lfloor (r+\log(1/\eps))/\log(2/(2-\alpha)) \rfloor$ we get that $\Pr_{D_\mu}[T^{\leq d}(x) \neq T(x)] \leq \eps$.
\end{theorem}
\begin{proof}
Our proof is by induction on the pruning depth $d$. If $T$ is a leaf, the statement trivial since $T^{\leq d}(x) \equiv T(x)$ for any $d \geq 0$. For $d=0$ and $r \geq 1$, $2^{r-1} \cdot \left(1-\frac{\alpha}{2}\right)^0 \geq 1$. We now assume that the claim is true for all pruning depths $0,\ldots,d-1$.

At least one of the subtrees $T_0$ and $T_1$ has rank $r-1$. Assume, without loss of generality that this is $T_0$. Let $\xx_i$ be the label of the root node of $T$.
$$\Pr_{D_\mu}[T^{\leq d}(x) \neq T(x)] = (1-\mu_i) \Pr_{D_\mu}[T_0^{\leq d-1}(x) \neq T_0(x)] + \mu_i \cdot \Pr_{D_\mu}[T_1^{\leq d-1}(x) \neq T_1(x)]\ .$$
By our inductive hypothesis, $$\Pr_{D_\mu}[T_0^{\leq d-1}(x) \neq T_0(x)] \leq 2^{r-2} \cdot \left(1-\frac{\alpha}{2}\right)^{d-1}$$ and $$\Pr_{D_\mu}[T_0^{\leq d-1}(x) \neq T_0(x)] \leq 2^{r-1} \cdot \left(1-\frac{\alpha}{2}\right)^{d-1}\ .$$ Combining these we get that
\alequn{\Pr_{D_\mu}[T^{\leq d}(x) \neq T(x)] &\leq (1-\mu_i) 2^{r-2} \cdot \left(1-\frac{\alpha}{2}\right)^{d-1} + \mu_i \cdot 2^{r-1} \cdot \left(1-\frac{\alpha}{2}\right)^{d-1} \\&\leq \alpha \cdot 2^{r-2} \cdot \left(1-\frac{\alpha}{2}\right)^{d-1} + (1-\alpha) \cdot 2^{r-1} \cdot \left(1-\frac{\alpha}{2}\right)^{d-1} \\&= \frac{1}{1-\frac{\alpha}{2}}\left(\frac{\alpha}{2} + (1-\alpha)\right)  2^{r-1} \cdot \left(1-\frac{\alpha}{2}\right)^{d} = 2^{r-1} \cdot \left(1-\frac{\alpha}{2}\right)^{d}\ .}
\end{proof}
For the uniform distribution we get error of at most $\eps$ for  $d = (r+\log(1/\eps))/\log(4/3) < \frac{5}{2} (r+\log(1/\eps))$.

\eat{
\begin{remark}
\label{rem:smooth-dist-approx}
We remark that the theorem is true for any other way to replace internal nodes at depth $d$ by functions. The claim would also hold under significantly weaker conditions on the distribution. Namely, the proof can be easily modified to obtain the same result for every $D$ for which there exist some constants $\alpha \in (0,1/2]$ and $\beta > 0$ such that for all $k \leq n$, and every disjunction $c(x)$ of $k$ literals, $\Pr_D[c(x) = 1] \leq \beta \cdot (1-\alpha)^k$. This would give a bound of $d = \lfloor (r+\log(\beta/\eps))/\log(2/(2-\alpha)) \rfloor$ (moreover the condition on the distribution only needs to hold for $k\leq d$).
\end{remark}
}


An immediate corollary of Theorems \ref{th:pruning} and \ref{th:submod-lip-rank-const-intro} is that every submodular function can be $\eps$-approximated over the uniform distribution by a binary decision tree of depth $O(1/\eps^2)$ (Corollary~\ref{cor:submod-tree-approx-intro}).

\citet{KushilevitzMansour:93} showed that the spectral $\ell_1$ norm of a decision tree of size $s$ is at most $s$.
\eat{
\begin{lemma}[\citep{KushilevitzMansour:93}]
Let $T$ be a $[0,1]$-valued decision tree of size $s$. Then
$$\|\hat{T}\|_1 = \sum_{S \subseteq [n]} |\hat{T}(S)| \leq s.$$
\end{lemma}}
Therefore we can immediately conclude that:
\begin{corollary}
\label{cor:submod-spectral-approx}
Let $f:\zon \rightarrow [0,1]$ be a submodular function and $\eps > 0$. There exists a function $p:\zon \rightarrow [0,1]$  such that $\|p - f\|_2 \leq \eps$ and $\|\hat{p}\|_1 = 2^{O(1/\eps^2)}$.
\end{corollary}


\section{Applications}

In this section, we give several applications of our structural results to the problem of learning submodular functions.

\subsection{PAC Learning}
In this section we present our results on learning in the PAC model. We first show how to find $2^{O(1/\eps^2)}$ variables that suffice for approximating any submodular function using random examples alone.
Using a fairly standard argument we first show that for any function $f$ that is close to a function of low polynomial degree and low spectral $\ell_1$ norm (which is satisfied by submodular functions) variables sufficient for approximating $f$ can be found by looking at significant Fourier coefficients of $f$ (the proof is in App.~\ref{app:pac-learning})
\begin{lemma}
\label{lem:junta-from-l1}
Let $f:\zon \rightarrow [0,1]$ be any function such that there exists a function $p$ of Fourier degree $d$ and spectral $\ell_1$ norm $\|\hat{p}\|_1 = L$ for which $\|f- p\|_2 \leq \eps$. Define
$$J = \{ i \cond \exists S; i\in S, |S| \leq d \mbox{ and } |\hat{f}(S)| \geq \eps^2/L \}  .$$
Then $|J| \leq d \cdot L^2/\eps^4$ and there exists a function $p'$ of Fourier degree $d$ over variables in $J$ such that $\|f- p\|_2 \leq 2 \eps$.
\end{lemma}
\begin{proof}
Let $$\cS = \{ S \cond |S| \leq d \mbox{ and } |\hat{f}(S)| \geq \eps^2/L \}. $$
By Parseval's identity, there are at most $L^2 / \eps^4$ sets in $\cS$.
Clearly, $J$ is the union of all the sets in $\cS$. Therefore, the bound on the size of $J$ follows immediately from the fact that each set $S \in \cS$ has size at most $d$.

Let $p'$ be the projection of $p$ to the subspace of $\{ \chi_S: S \in \cS\}$, that is $p' = \sum_{S \in \cS} \hat{p}(S) \chi_S$.
Now using Parseval's identity we get that
 $$\|f- p\|_2^2 = \sum_{S\subseteq [n]} (\hat{f}(S)-\hat{p}(S))^2\ . $$
 Now we observe that for any $S$, $|\hat{f}(S)-\hat{p}(S)| < |\hat{f}(S)-\hat{p'}(S)|$ can happen only when $S \not\in \cS$ in which case $\hat{p'}(S) = 0$ and $|\hat{f}(S)| \leq \eps^2/L$.

 $|\hat{p}(S)| \leq 2|\hat{f}(S)|$; hence only when $|\hat{p}(S)| \leq 2 \eps^2/L$.
 In this case, $$(\hat{f}(S)-\hat{p'}(S))^2 -  (\hat{f}(S)-\hat{p}(S))^2 = 2 \hat{f}(S) \hat{p}(S)  - (\hat{p}(S))^2 \leq 2 \hat{f}(S) \hat{p}(S) \leq 2 |\hat{p}(S)| \cdot \eps^2/L\ .$$
 Therefore,
$$ \|f- p'\|_2^2 - \|f- p\|_2^2 = \sum_{S} (\hat{f}(S)-\hat{p'}(S))^2 -  (\hat{f}(S)-\hat{p}(S))^2 \leq \frac{2\eps^2}{L} \sum_{S} |\hat{p}(S)|
  \leq \frac{2\eps^2}{L} \cdot \|\hat{p}\|_1  = 2\eps^2. $$
This implies that $\|f- p'\|_2^2 \leq 3\eps^2$.
\end{proof}


The second and crucial observation that we make is a connection between Fourier coefficient of $\{i,j\}$ of a submodular function and sum of squares of all Fourier coefficients that contain $\{i,j\}$.
\begin{lemma}
\label{lem:upp-bound-sum}
Let $f:\zon \rightarrow [0,1]$ be a submodular function and $i,j \in [n]$, $i\neq j$.
\equn{|\hat{f}(\{i,j\})| \geq \frac12 \sum_{S \ni i,j} (\hat{f}(S))^2  .}
\end{lemma}
\begin{proof}
\equn{|\hat{f}(\{i,j\})| =^{(a)} \frac14 |\E_\U[\partial_i \partial_j f]| =^{(b)} \frac14 \E_\U[|\partial_i \partial_j f|] \geq^{(c)} \frac18 \E_\U\left[\left(\partial_i \partial_j f\right)^2\right] =^{(a)} 2 \sum_{S \ni i,j} (\hat{f}(S))^2 .}
Here, $(a)$ follows from the basic properties of the Fourier spectrum of partial derivatives (see Sec.~\ref{sec:prelims});
$(b)$ is implied by second partial derivatives of a submodular function being always non-positive; and $(c)$ follows from $|\partial_i \partial_j f|$ having range $[0,2]$ whenever $f$ has range $[0,1]$.
\end{proof}

We can now easily complete the proof of Lemma~\ref{lem:find-inf-vars-intro}.
\begin{proof}[Proof of Lemma~\ref{lem:find-inf-vars-intro}]
The proof relies on two simple observations. The first one is that Lemma \ref{lem:junta-from-l1} implies that the set of indices $I_\gamma = \{ i \cond \exists S \ni i, |\hat{f}(S)| \geq \gamma \}$ satisfies the conditions of Lemma \ref{lem:find-inf-vars-intro} for some $\gamma = 2^{-O(1/\eps^2)}$.

Now if $i \in I_\gamma$ then either $|\hat{f}(\{i\})| \geq \gamma$ or,
exists $j\neq i$, such that for some $S' \ni i,j$, $|\hat{f}(S')| \geq \gamma$. In the latter case $\sum_{S \ni i,j} (\hat{f}(S))^2 \geq \gamma^2$. By Lemma \ref{lem:upp-bound-sum} we can conclude that then $|\hat{f}(\{i,j\})| \geq 2\gamma^2$.

This suggests the following simple algorithm for finding $J$. Estimate degree 1 and 2 Fourier coefficients of $f$ to accuracy $\gamma^2/2$ with confidence at least $5/6$ using random examples (note that $\gamma <1/2$ and hence degree-1 coefficients are estimated with accuracy at least $\gamma/4$. Let $\tilde{f}(S)$ for $S\subseteq [n]$ of size 1 or 2 denote the obtained estimates. We define
$$J = \left\{ i \cond \exists j\in [n], |\tilde{f}(\{i,j\})| \geq 3\gamma^2/2 \right\}\ . $$
If the estimates are correct, then clearly, $I_\gamma \subseteq J$. At the same time, $J$ contains inly indices which belong to a Fourier coefficient of magnitude at least $\gamma^2$ and degree at most $2$. By Parseval's identity, $|J| \leq 2 \|f\|_2^2/\gamma^4 = 2^{O(1/\eps^2)}$.

Finally, to bound the running time we observe that, by Chernoff bounds, $O(\log (n)/\gamma^4) = \log (n) \cdot 2^{O(1/\eps^2)} $ random examples are sufficient to obtain the desired estimates with confidence of $5/6$. The estimation of the coefficients can be done in $n^2 \log(n) \cdot 2^{O(1/\eps^2)}$ time.
\end{proof}

Now given a set $J$ that was output by the algorithm in Lemma \ref{lem:find-inf-vars-intro} one can simply run the standard low-degree algorithm of \citet{LinialMN:93} over variables with indices in $J$ to find a linear combination of parities of degree $O(1/\eps^2)$, $\eps$-close to $f$. Note that we need to find coefficients of at most $|J|^{O(1/\eps^2)} \leq \min\{2^{O(1/\eps^4)},n^{O(1/\eps^2)}\}$ parities. This immediately implies Theorem~\ref{th:pac-learn-submod-l1-intro}.
\eat{
\begin{theorem}(Theorem~\ref{th:pac-learn-submod-l1-intro} restated)
\label{th:pac-learn-submod-l1}
Let $\C_s$ be the set of all submodular functions from $\zon$ to $[0,1]$. There exists an algorithm $\A$ that given $\eps > 0$ and access to random uniform examples of any $f \in \C_s$, with probability at least $2/3$, outputs a function $h$, such that $\|f-h\|_2 \leq \epsilon$. Further, $\A$ runs in time $\tilde{O}(n^2) \cdot 2^{O(r/\eps^2)}$ and uses $2^{O(r/\eps^{2})} \log n$ examples for $r =  \min\{n, 1/\eps^2\}$.
\end{theorem}
}

\subsection{Agnostic learning with value queries}
Our next application is agnostic learning of submodular functions over the uniform distribution with value queries. We give two versions of the agnostic learning algorithm one based on $\ell_1$ and the other based on $\ell_2$ error. We note that, unlike in the PAC setting where small $\ell_2$ error also implied small $\ell_1$ error, these two versions are incomparable
and are also based on different algorithmic techniques. The agnostic learning techniques we use are not new but we give attribute-efficient versions of those techniques using an attribute-efficient agnostic learning of parities from \citep{Feldman:07jmlr}.

For the $\ell_2$ agnostic learning algorithm we need a known observation (\eg \citep{GopalanKK:08}) that the algorithm of \citet{KushilevitzMansour:93} can be used to obtain agnostic learning relative to $\ell_2$-norm of all functions with spectral $\ell_1$ norm of $L$ in time $\poly(n,L,1/\eps)$ (we include a proof in App.~\ref{app:att-eff-km}).
\eat{
\begin{theorem}[\citep{KushilevitzMansour:93}]
\label{th:km-gkk}
For $L>0$, we define $\C_L$ as $\{ p(x) \cond \|\hat{p}\|_1 \leq L\}$. There exists an algorithm $\A$ that given $\eps > 0$ and access to value queries for any real-valued $f:\zon\rightarrow [-1,1]$, with probability at least $2/3$, outputs a function $h$, such that $\|f-h\|_2 \leq \Delta + \epsilon$, where $\Delta = \min_{p\in \C_L}\{\|f-p\|_2\}$. Further, $\A$ runs in time $\poly(n,L,1/\eps)$.
\end{theorem}
}
We also observe that in order to learn agnostically decision trees of depth $d$ it is sufficient to restrict the attention to significant Fourier coefficients of degree at most $d$. We can exploit this observation to improve the number of value queries used for learning by using the attribute-efficient agnostic parity learning from \citep{Feldman:07jmlr} in place of the KM algorithm. Specifically, we first prove the following attribute-efficient version of agnostic learning of functions with low spectral $\ell_1$-norm (the proof appears in App.~\ref{app:att-eff-km}).

\begin{theorem}
\label{th:km-att-eff}
For $L>0$, we define $\C_L^d$ as $\{ p(x) \cond \|\hat{p}\|_1 \leq L \mbox{ and }\degree(p) \leq d \}$. There exists an algorithm $\A$ that given $\eps > 0$ and access to value queries for any real-valued $f:\zon\rightarrow [-1,1]$, with probability at least $2/3$, outputs a function $h$, such that $\|f-h\|_2 \leq \Delta + \epsilon$, where $\Delta = \min_{p\in \C_L}\{\|f-p\|_2\}$. Further, $\A$ runs in time $\poly(n,L,1/\eps)$ and uses $\poly(d,\log(n),L,1/\eps)$ value queries.
\end{theorem}
Together with Cor.~\ref{cor:submod-spectral-approx} this implies Theorem~\ref{th:agn-learn-valueq-intro}.

\eat{
\begin{theorem}(Theorem~\ref{thm:agn-learn-valueq} restated)
\label{th:agn-learn-valueq}
Let $\C_s$ denote the class of all submodular functions from $\zon$ to $[0,1]$. There exists an algorithm $\A$ that given $\eps > 0$ and access to value queries of any real-valued $f$, with probability at least $2/3$, outputs a function $h$, such that $\|f-h\|_2 \leq \Delta + \epsilon$, where $\Delta = \min_{g\in \C_s}\{\|f-g\|_2\}$. Further, $\A$ runs in time $\poly(n, 2^{1/\eps^2})$ and using $\poly(\log n, 2^{1/\eps^2})$ value\footnote{For Boolean functions value queries are usually referred to as membership queries.} queries.
\end{theorem}
}
\eat{
\begin{remark}
\label{rem:return-junta}
We note that this algorithm returns a function that can depend on variable $i$ only if $i \in S$ such that $|\hat{f}(S)| = 2^{-O(1/\eps^2)}$ and $|S| = O(1/\eps^2)$. By Parseval's identity, there are $2^{O(1/\eps^2)}$ such coefficients each with at most $O(1/\eps^2)$ variables. In other words, the learning algorithm returns a $2^{O(1/\eps^2)}$-junta as a hypothesis.
\end{remark}
}
\citet{GopalanKK:08} give the $\ell_1$ version of agnostic learning for functions of low spectral $\ell_1$ norm. Together with Cor.~\ref{cor:submod-spectral-approx} this implies an $\ell_1$ agnostic learning algorithm for submodular functions using $\poly(n,2^{1/\eps^2})$ time and queries. There is no known attribute-efficient version of the algorithm of \citet{GopalanKK:08} and their analysis is relatively involved. Instead we use our approximate representation by decision trees to invoke a substantially simpler algorithm for agnostic learning of decision trees based on agnostic boosting \citep{KalaiKanade:09,Feldman:10ab}. In this algorithm it is easy to use attribute-efficient agnostic learning of parities \citep{Feldman:07jmlr} (restated in Th.~\ref{th:aewp}) to reduce the query complexity of the algorithm. Formally we give the following attribute-efficient algorithm for learning $[0,1]$-valued decision trees.
\begin{theorem}
\label{th:agn-learn-rv-dt}
Let $\DT_{[0,1]}(r)$ denote the class of all $[0,1]$-valued decision trees of rank-$r$ on $\zon$. There exists an algorithm $\A$ that given $\eps > 0$ and access to value queries of any $f: \zon \rightarrow \zo$, with probability at least $2/3$, outputs a function $h: \zon \rightarrow [0,1]$, such that $\|f - h\|_1 \leq \Delta + \eps$, where $\Delta = \min_{g\in \DT_{[0,1]}(r)}\{\|f - g\|_1\}$. Further, $\A$ runs in time $\poly(n,2^r,1/\eps)$ and uses $\poly(\log{n},2^r,1/\eps)$ value queries.
\end{theorem}

Combining Theorems \ref{th:agn-learn-rv-dt} and \ref{th:submod-lip-rank-const-intro} gives the following agnostic learning algorithm for submodular functions (the proof is in App.~\ref{app:att-eff-km}).
\begin{theorem}
\label{th:agn-learn-l1-valueq}
Let $\C_s$ denote the class of all submodular functions from $\zon$ to $[0,1]$. There exists an algorithm $\A$ that given $\eps > 0$ and access to value queries of any real-valued $f$, with probability at least $2/3$, outputs a function $h$, such that $\|f-h\|_1 \leq \Delta + \epsilon$, where $\Delta = \min_{g\in \C_s}\{\|f-g\|_1\}$. Further, $\A$ runs in time $\poly(n, 2^{1/\eps^2})$ and using $\poly(\log n, 2^{1/\eps^2})$ value queries.
\end{theorem}

\section{Lower Bounds}
\label{sec:lower-bounds}
\subsection{Computational Lower Bounds for Agnostic Learning of Submodular Functions}
In this section we show that the existence of an algorithm for agnostically learning even \emph{monotone and symmetric\footnote{
In this context, we call a function $f:\zo^n \rightarrow \R$ symmetric if $f(x)$ depends only on $\sum x_i$.
This is different from the notion of a symmetric set function, which usually means the condition $f(S) = f(\bar{S})$.}
submodular functions} (i.e. concave functions of $\sum x_i$) to an accuracy of any $\epsilon > 0$ in time $n^{o({1}/{\epsilon^{2/3}})}$ would yield a faster algorithm for \textit{learning sparse parities with noise} (SLPN from now) which is a well known and notoriously hard problem in computational learning theory.

We begin by stating the problems of Learning Parities with Noise (LPN) and its variant, learning sparse parities with noise (SLPN). We say that random examples of a function $f$ have noise of rate $\eta$ if the label of a random example equals $f(x)$ with probability $1 - \eta$ and $-f(x)$ with probability $\eta$.
\begin{problem}[Learning Parities with Noise]
For $\eta \in (0,1/2)$, the problem of learning parities with noise $\eta$ is the problem of finding (with probability at least $2/3$) the set $S \subseteq [n]$, given access to random examples with noise of rate $\eta$ of parity function $\chi_S$. For $k\leq n$ the learning of $k$-sparse parities with noise $\eta$ is the same problem with an additional condition that $|S| \leq k$.
\end{problem}
The best known algorithm for the LPN problem with constant noise rate is by \citet{BlumKW:03} and runs in time $2^{O(n/\log{n})}$. The fastest known algorithm for learning $k$-sparse parities with noise $\eta$ is a recent breakthrough result of \citet{ValiantG:12}  which runs in time $O(n^{0.8k} \poly(\frac{1}{1-2\eta}))$.

\citet{KalaiKMS:08} and \citet{Feldman:12jcss} prove hardness of agnostic learning of majorities and conjunctions, respectively, based on correlation of concepts in these classes with parities. In both works it is implicit that if for every set $S \subseteq [n]$, a concept class $\C$ contains a function $f_S$ that has significant correlation with $\chi_S$ (or $\widehat{f_S}(S)$) then learning of parities with noise can be reduced to agnostic learning of $\C$. We now present this reduction in a general form.


\begin{lemma}
\label{lem:cor2lpn}
Let $\C$ be a class of functions mapping $\zo^n$ into $[-1,1]$. Suppose, there exist $\gamma >0$ and $k \in \N$ such that for every $S \subseteq [n]$, $|S| \leq k$, there exists a function, $f_S \in \C$, such that $|\widehat{f_S}(S)|\geq \gamma$. If there exists an algorithm $\A$ that learns the class $\C$ agnostically to accuracy $\epsilon$ in time $T(n, \frac{1}{\epsilon})$ then, there exists an algorithm $\A'$ that learns $k$-sparse parities with noise $\eta \leq 1/2$ in time $\poly(n,\frac{1}{(1-2\eta)\gamma}) + 2 T(n, \frac{2}{(1-2\eta)\gamma})$.
\end{lemma}

\begin{proof}
Let $\chi_S$ be the target parity with $|S| \leq k$. We run algorithm $\A'$ with $\epsilon = (1-2\eta) \gamma/2$ on the noisy examples and let $h$ be the hypothesis it outputs. We also run algorithm $\A'$ with $\epsilon = (1-2\eta) \gamma/2$ on the negated noisy examples and let $h'$ be the hypothesis it outputs.

Now let $f_S \in \C$ be the function such that $|\widehat{f_S}(S)|\geq \gamma$. Assume without loss of generality that $\widehat{f_S}(S) \geq \gamma$ (otherwise we will use the same argument on the negation of $f_S$). Let $\cal{N}^\eta$ denote the distribution over noisy examples.

For any function $f: \zo^n\rightarrow [-1,1]$,
\alequ{
\E_{(x,y) \sim \cal{N}^\eta} [ |f(x) - y| ] &= (1-\eta) \E_{x \sim \U} [ |f(x) - \chi_S(x)| ] + \eta \E_{x \sim \U} [ |f(x) + \chi_S(x)|] \nonumber \\
&= (1-\eta) \E_{x \sim \U} [ \chi_S(x) (\chi_S(x) - f(x)) ] + \eta \E_{x \sim \U} [ \chi_S(x) (\chi_S(x) + f(x)) \nonumber \\ &= 1 + (1 - 2\eta) \hat{f}(S) \label{eq:agn-error}.}

This implies that $$\E_{(x,y) \sim \cal{N}^\eta} [ |f_S(x) - y| ] = 1 + (1 - 2\eta) \widehat{f_S}(S) \geq 1+(1-2\eta)\gamma .$$

By the agnostic property of $\A$ with $\epsilon = (1-2\eta) \gamma/2$, the returned hypothesis $h$ must satisfy $$\E_{(x,y) \sim \cal{N}^\eta} [ |h(x) - y| ]   \geq 1+(1-2\eta)\gamma - (1-2\eta)\gamma/2 \geq 1+ (1-2\eta)\gamma/2.$$
By equation (\ref{eq:agn-error}) this implies that $\hat{h}(S) \geq \gamma/2$.

We can now use the algorithm of \citet{GoldreichLevin:89} (or a similar one) algorithm to find all sets with a Fourier coefficient of at least $\gamma/4$ (with accuracy of $\gamma/8$). This can be done in time polynomial in $n$ and $1/\gamma$ and will give a set of coefficients of size at most $O(1/\gamma^2)$ which contains $S$. By testing each coefficient in this set on $O( (1-2\eta)^{-2} \log{(1/\gamma)})$ random examples and choosing the one with the best agreement we find $S$.
\end{proof}

We will now show that there exist monotone symmetric submodular functions that have high correlation with the parity functions (the proof is in Appendix \ref{sec:prove-correlation}).
\begin{lemma}[Correlation of Monotone Submodular Functions with Parities]
\label{lem:corr-parity-submod}
Let $S \subseteq [n]$ such that $|S| = s$ for some $s \in [n]$. Then,
there exists a \emph{monotone} symmetric submodular function $H_S:\zo^n \rightarrow [0,1]$ such that $H_S$ depends only on coordinates in $S$ and $|\langle \chi_S, H_S \rangle| = \Omega(s^{-3/2})$. \label{corsubmodular}
\end{lemma}

Combining this result with Lemma \ref{lem:cor2lpn}, we now obtain the following reduction of SLPN to agnostically learning monotone submodular functions:
\begin{theorem}[Theorem~\ref{th:agnostic-hardness-intro} restated]
If there exists an algorithm that agnostically learns all monotone submodular functions with range $[0,1]$ to $\ell_1$ error of $\eps>0$ in time $T(n,1/\eps)$ then there exists an algorithm that learns $(\eps^{-2/3})$-sparse parities with noise of rate $\eta < 1/2$ in time $\poly(n,1/(\eps(1-2\eta))) + 2T(n, c/\eps(1-2\eta))$ for some fixed constant $c$.
\end{theorem}
\begin{proof}
Consider all the monotone submodular functions $R_S$ for every $S \subseteq [n]$, $|S| \leq k=\eps^{-2/3}$. Then, $|\langle \chi_S, H_S \rangle | = \Omega(k^{-3/2}) = \Omega(\epsilon)$ by Lemma \ref{corsubmodular}. Thus, using $\gamma = \Omega(\epsilon)$ in Lemma \ref{lem:cor2lpn} we obtain the claim.
\end{proof}

\subsection{Information-Theoretic Lower Bound for PAC-learning Submodular Functions}
In this section we show that any algorithm that PAC-learns monotone submodular functions to accuracy $\epsilon$ must use $2^{\Omega(\epsilon^{-2/3})}$ examples. The idea is to show that the problem of learning the class all boolean functions on $k$ variables to any constant accuracy can be reduced to the problem of learning submodular functions on $2t = k + \lceil \log{k} \rceil + O(1)$ variables to accuracy $O(\frac{1}{t^{3/2}})$. Any algorithm that learns the class of all boolean functions on $k$ variables to accuracy $1/4$ requires at least $\Omega(2^k)$ bits of information. In particular at least that many random examples or value queries are necessary.


Before we go on the present the reduction, we need to make a quick note regarding a slight abuse of notation: In the lemma below, we will encounter uniform distributions on hypercubes of two different dimensions. We will, however, still represent uniform distributions on either of them by $\U$ (with the meaning clear from the context).

\begin{lemma}
Let $f:\zo^k \rightarrow \zo$ be any boolean function. Let $t > 0$ be such that ${{2t} \choose t} \geq 2^k > {{2t-2} \choose {t-1}}$ (thus $4\cdot 2^k > {{2t} \choose t} \geq 2^k$). There exists a monotone submodular function $h:\zo^{2t} \rightarrow [0,1]$ such that:
\begin{enumerate}
\item $h$ can be computed at any point $x \in \zo^{2t}$ in at most a single query to $f$ and in time $O(t)$.
\item Let $\alpha = \frac{2^k\cdot \sqrt{t}}{2^{2t}} = \theta(1)$. Given any function $g:\zo^{2t} \rightarrow \R$ that approximates $h$, that is, $\E_{x \sim \U} [|h(x) - g(x)|] \leq \alpha \cdot \frac{\epsilon}{8t^{3/2}}$, there exists a boolean function $\tilde{f}:\zo^k \rightarrow \zo$ such that $\E_{x \sim \U} [ |\tilde{f}(x) - f(x)|] \leq \epsilon$ and $\tilde{f}$ can be computed at any point $x \in \zo^k$, with a single query to $g$ and in time $O(t)$.
\end{enumerate}\label{embed}
\end{lemma}

\begin{proof}
We first give a construction for the function $h$. It will be convenient first to define another function $\tilde{h}:\zo^{2t} \rightarrow [0,1]$ and then modify it to obtain $h$. Recall that for any $x$ and $S \subseteq [2t]$, $w_S(x) = \sum_{ i \in S} \frac{1}{2}(x_i+1)$. The function $\tilde{h}$ would be the same as the function $H_S$ defined in the proof of Lemma \ref{corsubmodular}.

\[
\tilde{h}(x)=\left\{\begin{array}{cl}
	w_{[2t]}(x)/t & w_{[2t]}(x) \leq t  \\
          1  & w_{[2t]}(x) > t\\
	   \end{array}\right.
\]

We will now define $h$ using $\tilde{h}$ and $f$. The key idea is that even if we lower the value of $\tilde{h}$ at any $x$ with $w_{[2t]} (x) = k$ by $\frac{1}{2t}$, the resulting function remains submodular. Thus, we embed the boolean function $h$ by modifying the values of $\tilde{h}$ at only the points in the middle layer ($w_{[2t]}(x) = t$).

Let $s = {{2t} \choose t}$. Let $M_{2t} = \{x  \in \zo^{2t}\mid w_{[2t]}(x) = t\}$ and $M_{k} = \{y \in \zo^k \}$ and $s \geq 2^k$. Let $\beta:M_{k} \rightarrow M_{2t}$ be an injective map of $M_{k}$ into $M_{2t}$ such that both $\beta$ and $\beta^{-1}$ (whenever it exists) can be computed in time $O(t)$ at any given point. Such a map exists, as can be seen by imposing lexicographic ordering on $M_{2t}$ and $M_{k}$ and defining $\beta(x)$ for $x \in M_{2t}$ to be the element in $M_{k}$ with the same position in the ordering as that of $x$. For each $x \in \zo^{2t}$, let $h$ be defined by:

\[
h(x)=\left\{\begin{array}{cl}
	\tilde{h}(x) & w_{[2t]}(x) \neq t  \\
          (1-\frac{1}{2t}) & w_{[2t]}(x) = t \text{, } \beta^{-1}(x) \text{ exists and } f(\beta^{-1}(x)) = 0  \\
          1 & w_{[2t]}(x) = t \text{, } \beta^{-1}(x) \text{ exists and } f(\beta^{-1}(x)) = 1 \\
          1 & \text{ otherwise }\\
	   \end{array}\right.
\]

Notice that given any $x \in \zo^{2t}$ the value of $h(x)$ can be computed by a single query to $f$. Further, observe that $\tilde{h}$ is monotone and $h$ is obtained by modifying $\tilde{h}$ only on points in $M_{2t}$ and by  at most $\frac{1}{2t}$, which ensures that for any $x \leq y$ such that $w_{[2t]}(x) < w_{[2t]}(y)$, $h(x) \leq h(y)$. Moreover, $M_{2t}$ forms an antichain in the partial order on $\zo^n$ and thus no two points in $M_{2t}$ are comparable. This proves that $h$ is monotone. \\
Suppose, now that $g:\zo^{2t} \rightarrow \R$ is such that $\E_{x \sim \U} [|h(x) -g(x)|] \leq \alpha \cdot \frac{\epsilon}{8t^{3/2}}$.

Define $g_b:\zo^{2t} \rightarrow \zo$ so that $$\forall x \in \zo^{2t} \text{, } g_b(x) = sign\left( g(x)  - (1-(1/4t))\right).$$ Finally, let $\tilde{f}:\zo^k \rightarrow \zo$ be such that for every $x \in \zo^k$ $\tilde{f}(x) = g_b(\beta(x))$.

Now,
$\E_{x \sim \U} [|\tilde{f}(x) - f(x)|] = 2 \Pr_{x \sim \U} [ \tilde{f}(x) \neq f(x)]$. For any $x \in \zo^k$,
$$ \tilde{f}(x) \neq f(x) \Leftrightarrow |g(\beta(x))-h(\beta(x))| \geq \frac{1}{4t}. $$
 Using that $\Pr_{y \sim \U} [ \beta^{-1}(y) \text{ exists }] = \frac{\alpha}{\sqrt{t}}$, we have: \begin{align*}
 \E_{y \sim \U}[|g(y) - h(y)|] &\geq \frac{1}{4t} \Pr_{y \sim \U}[ \beta^{-1}(y) \text{ exists and }\tilde{f}(\beta^{-1}(y) \neq f(\beta^{-1}(y)] \\&= \frac{1}{8t} \frac{\alpha}{\sqrt{t}} \E_{x \sim \U} [|\tilde{f}(x) - f(x)|].
 \end{align*}

Using $\E_{y \sim \U}[|g(y) - h(y)|] \leq \alpha \cdot \frac{\epsilon}{8\cdot (t)^{3/2}}$, we have: $ \E_{x \sim \U} [|\tilde{f}(x) - f(x)|] \leq \epsilon$.

Finally, we show that $h$ is submodular for any boolean function $f$. It will be convenient to switch notation and look at input $x$ as the indicator function of the set $S_x= \{ x_i \mid x_i = 1\}$. We will verify that for each $S \subseteq [n]$ and $i,j \notin S$,
\begin{equation}
h(S \cup \{i\}) - h(S) \geq h(S \cup \{i,j\})  - h(S \cup \{j\}).\label{submodularity}
 \end{equation}
Notice that $\tilde{h}$ is submodular, and $h = \tilde{h}$ on every $x$ such that $w_{[2t]}(x) \neq t$. Thus, we only need to check Equation \eqref{submodularity} for $S, i, j$ such that $|S| \in \{ t-2, t-1, t\}$.
We analyze these $3$ cases separately:
\begin{enumerate}
\item
 $\bm{|S| = t-1:}$ Notice that $h(S) = \tilde{h}(S) = 1-(1/t)$ and $h(S \cup \{i, j\}) = \tilde{h}( S \cup \{i, j\}) = 1$. Also observe that for any $f$, $h(S \cup \{i\})$ and $h(S \cup \{j\})$ are at least $(1-\frac{1}{2t})$. Thus, $h(S \cup \{i\}) + h(S \cup \{j\}) \geq 2 - \frac{1}{t} = h(S) + h(S \cup \{i,j\})$.

 \item
 $\bm{|S| = t-2:}$ In this case, $h(S) = (1-(2/t))$ and $h(S \cup \{i\}) = h(S \cup \{j\}) = (1-(1/t))$. In this case, the maximum value for any $f$, of $h(S \cup \{i,j\}) = 1$. Thus, $$h(S) +  h(S \cup \{i,j\})  \leq 2 - (2/t) = h(S \cup \{i\}) + h(S \cup \{j\}).$$
 \item
$\bm {|S| = t:}$ Here, $h(S \cup \{i\}) = h(S \cup \{j\}) =  h(S \cup \{i,j\}) = 1$. The maximum value of $h(S)$ for any $f$ is $1$. Thus, $$h(S) +  h(S \cup \{i,j\})  \leq 2 = h(S \cup \{i\}) + h(S \cup \{j\}).$$
 \end{enumerate}
This completes the proof that $h$ is submodular.
\end{proof}

We now have the following lower bound on the running time of any learning algorithm (even with value queries) that learns monotone submodular functions.

\begin{theorem}[Theorem~\ref{th:PAC-hardness-intro} restated]
Any algorithm that PAC learns all monotone submodular functions with range $[0,1]$ to $\ell_1$ error of $\eps>0$ requires $2^{\Omega(\eps^{-2/3})}$ value queries to $f$.
\end{theorem}
\begin{proof}
We borrow notation from the statement of Lemma \ref{embed} here. Given an algorithm that PAC learns monotone submodular functions on $2t$ variables, we describe how one can obtain a learning algorithm for all boolean function on $k$ variables with accuracy $1/4$.
Given an access to a boolean function $f:\zo^k \rightarrow \zo$, we can translate it into an access to a submodular function $h$ on $2t$ variables with an overhead of at most $O(t) = O(k)$ time using Lemma \ref{embed}. Using the PAC learning algorithm, we can obtain a function $g:\zo^{2t} \rightarrow \R$ that approximates $h$ within an error of at most $\alpha \cdot \frac{1}{8t^{3/2}}$ and Lemma \ref{embed} shows how to obtain $\tilde{f}$ from $g$ with an overhead of at most $O(t) = O(k)$ time such that $\tilde{f}$ approximates $f$ within $\frac{1}{4}$.
Choose $k =  \lceil \epsilon^{-2/3} \rceil$ and $t$ as described in the statement of Lemma \ref{embed}.
Now, using any algorithm that learns monotone submodular functions to an accuracy of $\epsilon >0$ we obtain an algorithm that learns all boolean functions on $k = \lceil \epsilon^{-2/3} \rceil$ variables to accuracy $1/4$.
\end{proof}


\appendix

\section{Attribute-efficient Agnostic Learning}
\label{app:att-eff-km}
In this section we give attribute-efficient versions of two agnostic learning algorithms: (1) the $\ell_2$-error agnostic learning of functions with low spectral $\ell_1$-norm and (2) $\ell_1$-error agnostic learning of (real-valued) decision trees. The algorithms are obtained using a simple combination of existing techniques with attribute-efficient weak agnostic parity learning from \citep{Feldman:07jmlr}. For the first algorithm we are not aware of published details of the analysis even without the attribute-efficiency.

We first state the attribute-efficient weak agnostic parity learning from \citep{Feldman:07jmlr}.
\begin{theorem}
\label{th:aewp}
There exists an algorithm $\WP$, that given an integer $d$, $\theta > 0$ and $\delta \in (0,1]$, access to value queries of any $f:\zon \rightarrow [-1,1]$ such that $|\hat{f}(S)| \geq \theta$ for some $S$, $|S| \leq d$, with probability at least $1-\delta$, returns $S'$, such that $|\hat{f}(S')| \geq \theta/2$ and $|S'| \leq d$. $\WP(d,\theta,\delta)$ runs in $\tilde{O}\left(nd^2\theta^{-2} \logd\right)$ time and asks $\tilde{O}\left(d^2\log^2{n}\cdot \theta^{-2} \logd\right)$ value queries.
\end{theorem}
Using $\WP$ we can find a set $\cS$ of subsets of $[n]$ such that (1) if $S \in \cS$ then $|\hat{f}(S)| \geq \theta/2$ and $|S|\leq d$; (2) if $|\hat{f}(S)| \geq \theta$ and $|S|\leq d$ then $S \in \cS$. The first property, implies that $|\cS| \leq 4/\theta^2$. With probability $1-\delta$, $\cS$ can be found in time polynomial in $1/\theta^2$ and the running time of $\WP(d,\theta, 4\delta/\theta^2)$. With probability at least $1-\delta$, each coefficient in $\cS$ can be estimated to within $\theta/4$ using a random sample of size $\tilde{O}(\logd/\theta^2)$. This gives the following low-degree version of the Kushilevitz-Mansour algorithm \citep{KushilevitzMansour:93}.
\begin{theorem}
\label{th:att-eff-collect-FT}
There exists an algorithm $\AEFT$, that given an integer $d$, $\theta > 0$ and $\delta \in (0,1]$, access to value queries of any $f:\zon \rightarrow [-1,1]$, with probability at least $1-\delta$, returns a function $h$ represented by the set of its non-zero Fourier coefficients such that
\begin{enumerate}
\item $\degree(h) \leq d$;
\item for all $S\subseteq [n]$ such that $|\hat{f}(S)| \geq \theta$ and $|S| \leq d$, $\hat{h}(S) \neq 0$;
\item for all $S\subseteq [n]$, if $|\hat{f}(S)|\leq \theta/2$ then $\hat{h}(S) = 0$;
\item if $\hat{h}(S) \neq 0$ then $|\hat{f}(S) - \hat{h}(S)| \leq \theta/4$.
\end{enumerate}
$\AEFT(d,\theta,\delta)$ runs in $\tilde{O}\left(nd^2\theta^{-2} \logd\right)$ time and asks $\tilde{O}\left(d^2 \log^2{n} \cdot \theta^{-2} \logd\right)$ value queries.
\end{theorem}

We now show that for $\theta = \eps^2/(2L)$, $\AEFT$ agnostically learns the class
 $$\C_L^d = \{ p(x) \cond \|\hat{p}\|_1 \leq L \mbox{ and }\degree(p) \leq d \}\ .$$
\begin{lemma}
\label{lem:km-to-agnostic}
For $L >0, \eps\in (0,1)$ and integer $d$, let $f:\zon \rightarrow [-1,1]$ and $h:\rightarrow \R$ be functions such that for $\theta = \eps^2/(2L)$,
\begin{enumerate}
\item $\degree(h) \leq d$;
\item for all $S\subseteq [n]$ such that $|\hat{f}(S)| \geq \theta$ and $|S| \leq d$, $\hat{h}(S) \neq 0$;
\item for all $S\subseteq [n]$, if $|\hat{f}(S)|\leq \theta/2$ then $\hat{h}(S) = 0$;
\item if $\hat{h}(S) \neq 0$ then $|\hat{f}(S) - \hat{h}(S)| \leq \theta/4$.
\end{enumerate}
Then for any $g \in \C_L^d$, $\|f-h\|_2 \leq \|f-g\|_2 + \eps$.
\end{lemma}
\begin{proof}
We show that for every $S \subseteq [n]$,
\equ{(\hat{f}(S)-\hat{h}(S))^2 \leq (\hat{f}(S)-\hat{g}(S))^2 + 2 \theta \cdot |\hat{g}(S)| = (\hat{f}(S)-\hat{g}(S))^2 + \frac{\eps^2 \cdot |\hat{g}(S)|}{L} \label{eq:coeff-bound}.}
First note that this would immediately imply that
\alequn{\|f-h\|_2^2 &= \sum_{S\subseteq [n]} (\hat{f}(S)-\hat{h}(S))^2 \leq \sum_{S\subseteq [n]} (\hat{f}(S)-\hat{g}(S))^2 + \frac{\eps^2 \cdot |\hat{g}(S)|}{L} = \|f-g\|_2^2 + \frac{\eps^2 \cdot \|\hat{g}\|_1}{L} \\ & \leq \|f-g\|_2^2 + \eps^2 \leq  (\|f-g\|_2 + \eps)^2 .}
To prove equation (\ref{eq:coeff-bound}) we consider two cases. If $\hat{h}(S) = 0$, then either $|S| > d$ or $|\hat{f}(S)| \leq \theta$. In the former case $\hat{g}(S) = 0$ and therefore equation (\ref{eq:coeff-bound}) holds. In the latter case:
$$(\hat{f}(S)-\hat{h}(S))^2 = (\hat{f}(S))^2 \leq (\hat{f}(S)-\hat{g}(S))^2 + 2 |\hat{f}(S)| \cdot |\hat{g}(S)| \leq (\hat{f}(S)-\hat{g}(S))^2 + 2 \theta \cdot |\hat{g}(S)|\ . $$

In the second case (when $\hat{h}(S) \neq 0$), we get that $|\hat{f}(S)| \geq \theta/2$ and $|\hat{f}(S) - \hat{h}(S)| \leq \theta/4$. Therefore, either $|\hat{g}(S)| \leq |\hat{f}(S)|/2$ and then $(\hat{f}(S)-\hat{g}(S))^2 \geq (\hat{f}(S))^2/4 \geq \theta^2/16$ or $|\hat{g}(S)| \geq |\hat{f}(S)|/2 \geq \theta/4$ and then $2 \theta \cdot |\hat{g}(S)| \geq \theta^2/2$. In both cases,
$$(\hat{f}(S)-\hat{h}(S))^2 \leq \frac{\theta^2}{16} \leq (\hat{f}(S)-\hat{g}(S))^2 + 2 \theta \cdot |\hat{g}(S)|\ .$$
\end{proof}
Theorem \ref{th:km-att-eff} is a direct corollary of Theorem \ref{th:att-eff-collect-FT} and Lemma \ref{lem:km-to-agnostic}.

The proof of Theorem \ref{th:agn-learn-l1-valueq} relies on agnostic learning of decision trees. We first give an attribute-efficient algorithm for this problem.
\begin{theorem}
\label{th:agn-learn-b-dt}
Let $\DT(r)$ denote the class of all Boolean decision trees of rank-$r$ on $\zon$. There exists an algorithm $\A$ that given $\eps > 0$ and access to value queries of any $f: \zon \rightarrow \zo$, with probability at least $2/3$, outputs a function $h: \zon \rightarrow \zo$, such that $\Pr_\U[f \neq h] \leq \Delta + \eps$, where $\Delta = \min_{g\in \DT(r)}\{\Pr_\U[f \neq g]\}$. Further, $\A$ runs in time $\poly(n,2^r,1/\eps)$ and uses $\poly(\log{n},2^r,1/\eps)$ value queries.
\end{theorem}
\begin{proof}
We first use Theorem \ref{th:pruning} to reduce the problem of agnostic learning of decision trees of rank at most $r$ to the problem of agnostic learning of decision trees of depth $\frac{5}{2}(r+\log{(2/\eps)})$ with error parameter $\eps/2$.
In \citep{Feldman:10ab} and \citep{KalaiKanade:09} it is shown that a distribution-specific agnostic boosting algorithm reduces the problem of agnostic learning decision trees of size $s$ with error $\eps'=\eps/2$ to that of weak agnostic learning of decision trees invoked $O(s^2/\eps'^2)$ times. It was also shown in those works that agnostic learning of parities with error of $\eps'/(2s)$ gives the necessary weak agnostic learning of decision trees. Further, as can be easily seen from the proof, for decision trees of depth $\leq d$ it is sufficient to agnostically learn parities of degree $\leq d$. In our case the size of the decision tree is $\leq 2^d = (2^{r+1}/\eps)^{5/2}$. We can use $\WP$ algorithm with error parameter $\eps'/(2s) \geq \eps^{7/2}/ 2^{\frac{5r}{2}+5}$ and degree $d$, to obtain weak agnostic learning of decision trees in time $\poly(n,2^r,1/\eps)$ and using $\poly(\log{n},2^r,1/\eps)$ value queries. This implies that agnostic learning of decision trees can be achieved in time $\poly(n,2^r,1/\eps)$ and using $\poly(\log{n},2^r,1/\eps)$ value queries.
\end{proof}

From here we can easily obtain an algorithm for agnostic learning of rank-$r$ decision trees with real-valued constants from $[0,1]$. We obtain it by using a simple argument (see \citep{FeldmanKothari:13covman} for a simple proof) that reduces learning of a real-valued function $g$ to learning of boolean functions of the form $g_\theta(x) = ``g(x) \geq \theta"$ (note that every $g: \zon \rightarrow [0,1]$, is $\eps$-close (in $\ell_1$ distance) to $g'(x) = \sum_{i \in \lfloor 1/\eps\rfloor} g_{i \eps}(x)$). We now observe that if $g$ can be represented as a decision tree of rank $r$, then for every $\theta$, $g_\theta$ can be represented as a decision tree of rank $r$. Therefore this reduction implies that agnostic learning of Boolean rank-$r$ decision trees gives agnostic learning of $[0,1]$-valued rank-$r$ decision trees. The reduction runs the Boolean version $2/\eps$ times with accuracy $\eps/2$ and yields the proof of Theorem \ref{th:agn-learn-rv-dt}.

\eat{
\begin{proof}[Outline]
We first use a standard decomposition of a real-valued function into a discretized sum of boolean threshold functions. In other words we use the fact that for every $g: \zon \rightarrow \{0,1,\ldots,k\}$, $g = \sum_{i \in [k]} g_k$, where $g_i$ is a Boolean function defined as $g_i(x) = ``g(x) \geq i"$. We now observe that is $g$ can be represented as a decision tree of depth $d$, then $g_i$ can be represented as a decision tree of depth $d$. Now it is sufficient to observe that for any distribution $D$:
\begin{enumerate}
\item For $f,g: \zon \rightarrow \{0,1,\ldots,k\}$, $\sum_{i\in [k]} \Pr_D[f_i \neq g_i] = \E_D[|f-g|]$;
\item For $f\zon \rightarrow \{0,1,\ldots,k\}$ and a set of $k$ Boolean functions $h_1,h_2,\ldots,h_k$,
$$ \E_D[|f- \sum_{i\in [k]} h_i|] \leq \sum_{i\in [k]} \Pr_D[f_i \neq h_i] .$$
\end{enumerate}
Combining these two observations allows us to reduce agnostic learning of $\{0,1,\ldots,k\}$-valued functions (with $\ell_1$ error) to agnostic learning of $k$ Boolean functions.
Agnostic learning of $k$ decision trees of depth $O(k+\log(1/\eps))$, with error of $\eps/(2k)$ can be done in time $\poly(n,2^k,1/\eps)$ and using $\poly(\log{n},2^k,1/\eps)$ value queries.
\end{proof}
}

\section{Learning Pseudo-Boolean Submodular Functions}
\label{sec:pseudo-boolean}
In a recent work, \citet{RaskhodnikovaYaroslavtsev:13soda} consider learning and testing of submodular functions taking values in the range $\{0,1,\ldots,k\}$. The error of a hypothesis in their framework is the probability that the hypothesis disagrees with the unknown function (hence it is referred to as {\em pseudo-Boolean}).
For this restriction they give a $\poly(n) \cdot k^{O(k \log{k/\eps})}$-time PAC learning algorithm using value queries.

As they observed, error $\eps$ in their model can also be obtained by learning the function scaled to the range $\{0,1/k,\ldots,1\}$ with $\ell_1$ error of $\eps/k$ (since for two functions with that range $\E[|f-h|] \leq \eps/k$ implies that $\Pr[f \neq h] \leq \eps$). Therefore our structural results can also be interpreted in their framework directly. We now show that even stronger results are implied by our technique.

The first observation is that a $\frac{1}{k+1/3}$-Lipschitz function with the range $\{0,1/k,\ldots,1\}$ is a constant. Therefore Theorem \ref{th:submod-lip-rank} implies an exact representation of submodular functions with range $\{0,1,\ldots,k\}$ by decision trees of rank $\leq \lfloor 2k+2/3 \rfloor = 2k$ with constants from $\{0,1/k,\ldots,1\}$ in the leafs. We note that this representation is incomparable to $2k$-DNF representation which is the basis of results in \citep{RaskhodnikovaYaroslavtsev:13soda}.

We can also directly combine Theorems \ref{th:submod-lip-rank} and \ref{th:pruning} to obtain the following analogue of Corollary \ref{cor:submod-tree-approx-intro}.
\begin{theorem}
\label{th:submod-tree-approx-kval}
Let $f:\zon \rightarrow \{0,1,\ldots,k\}$ be a submodular function and $\eps > 0$. There exists a $\{0,1,\ldots,k\}$-valued decision tree $T$ of depth $d  = 5(k+\log{(1/\eps)})$ such that $\Pr_\U[T \neq f] \leq \eps$. In particular, $T$ depends on at most $2^{5k}/\eps^5$ variables and $\|\hat{T}\|_1 \leq 2k \cdot 2^{5k}/\eps^5$.
\end{theorem}
These results improve on the spectral norm bound of $k^{O(k \log{k/\eps})}$ from \citep{RaskhodnikovaYaroslavtsev:13soda}.
In a follow-up (independent of this paper) work \citet{BlaisOSY:13manu} also obtained an approximation of discrete submodular functions by juntas. They prove that every submodular function $f$ of range of size $k$ is $\eps$-close to a function of $(k \log(k/\eps))^{O(k)}$ variables and give an algorithm for testing submodularity using $(k \log(1/\eps))^{\tilde{O}(k)}$ value queries. Note that our bound has a better dependence on $k$ but worse on $\eps$ (the bounds have the same order when $\eps = k^{-k}$).

As in the general case, these structural results can be used to obtain learning algorithms in this setting. It is natural to require that learning algorithms in this setting output a $\{0,1,\ldots,k\}$-valued hypothesis. We observe that the algorithm in Theorem \ref{th:agn-learn-rv-dt} can be easily modified to return a $\{0,1/k,\ldots,1\}$-valued function when it is applied for learning $\{0,1/k,\ldots,1\}$-valued functions. This is true since the proof of Theorem \ref{th:agn-learn-rv-dt} (see Section \ref{app:att-eff-km} discretizes the target function and reduces the problem to learning of Boolean functions. $\{0,1/k,\ldots,1\}$-valued functions are already discretized. With this exact discretization the output of the agnostic algorithm is a sum of $k$ Boolean hypotheses, and in particular is a $\{0,1/k,\ldots,1\}$-valued function. This immediately leads to the following algorithm for agnostic learning of $\{0,1,\ldots,k\}$-valued submodular functions.
\begin{theorem}
\label{th:agn-learn-valueq-kval}
Let $\C_s^k$ denote the class of all submodular functions from $\zon$ to $\{0,1,\ldots,k\}$. There exists an algorithm $\A$ that given $\eps > 0$ and access to value queries of any $f: \zon \rightarrow \{0,1,\ldots,k\}$, with probability at least $2/3$, outputs a function $h$ with the range in $\{0,1,\ldots,k\}$, such that $\E_\U[|f - h|] \leq \Delta + \epsilon$, where $\Delta = \min_{g\in \C_s^k}\{\E_\U[|f - g|]\}$. Further, $\A$ runs in time $\poly(n,2^k,1/\eps)$ and uses $\poly(\log{n},2^k,1/\eps)$ value queries.
\end{theorem}
This improves on $\poly(n) \cdot k^{O(k \log{k/\eps})}$-time and queries algorithm with the same guarantees which is implied by the spectral bounds in \citep{RaskhodnikovaYaroslavtsev:13soda}. We remark that the guarantee of this algorithm implies PAC learning with disagreement error (since for integer valued hypotheses $\ell_1$-error upper-bounds the disagreement error). At the same time the guarantee is not agnostic in terms of the disagreement error\footnote{In \citep{RaskhodnikovaYaroslavtsev:13soda} it was mistakenly claimed that the application of the algorithm of \citet{GopalanKK:08} gives agnostic guarantee for the disagreement error.} (but only for $\ell_1$-error).

The structural results also imply that when adapted to this setting our PAC learning algorithm in Theorem \ref{th:pac-learn-submod-l1-intro} leads to the following PAC learning algorithm in this setting.
\begin{theorem}
\label{th:pac-learn-submod-kval}
There exists an algorithm $\A$ that given $\eps > 0$ and access to random uniform examples of any $f \in \C_s^k$, with probability at least $2/3$, outputs a function $h$, such that $\Pr_\U[f \neq h]\leq \epsilon$. Further, $\A$ runs in time $\tilde{O}(n^2) \cdot 2^{O(k^2+ \log^2(1/\eps))}$ and uses $2^{O(k^2+ \log^2(1/\eps))} \log n$ examples.
\end{theorem}
For learning from random examples alone, previous structural results imply only substantially weaker bounds:
($\poly(n^k,1/\eps)$ in \citep{RaskhodnikovaYaroslavtsev:13soda}).

Finally, we show that the combination of approximation by a junta and exact representation by a decision tree lead to a proper PAC learning algorithm for pseudo-Boolean submodular functions in time $\poly(n) \cdot 2^{O(k^2 + k\log(1/\eps))}$ using value queries. Note that, for the general submodular functions our results imply only a doubly-exponential time algorithm (with singly exponential number of random examples).
\begin{theorem}
\label{th:proper-learn-valueq-kval}
Let $\C_s^k$ denote the class of all submodular functions from $\zon$ to $\{0,1,\ldots,k\}$. There exists an algorithm $\A$ that given $\eps > 0$ and access to value queries of any $f \in \C_s^k$, with probability at least $2/3$, outputs a {\em submodular} function $h$, such that $\Pr[f\neq h] \leq \epsilon$. Further, $\A$ runs in time $\poly(n, 2^{k^2+k\log{1/\eps}})$ and uses $\poly(\log n, 2^{k^2+k\log{1/\eps}})$ value queries.
\end{theorem}
\begin{proof}[Proof Outline:]
In the first step we identify a small set of variables $J$ such that there exists a function that depends only on variables indexed by $J$ and is $\eps/3$ close to $f$. This can be achieved (with probability at least $2/3$) by using the algorithm in Lemma \ref{lem:find-inf-vars-intro} (with bounds adapted to this setting) to obtain a set of size $\poly(2^k/\eps)$. Now let $\U_J$ represent a uniform distribution over $\zo^J$ and $\U_{\bar{J}}$ represent the uniform distribution over $\bar{J} = [n]\setminus J$. Let $g$ be the function that depends only on variables in $J$ and is $\eps/3$ close to $f$. Then,
\equn{\Pr_\U[f(x) \neq g(x)]  = \E_{z \sim \U_{\bar{J}}} \left[\Pr_{y \sim \U_J}[f(y,z) \neq g(y,\bar{0})] \right] \leq \eps/3\ .}

By Markov's inequality, this means that with probability at least $1/2$ over the choice of $z$ from $\zo^{\bar{J}}$,
$\Pr_{y \sim \U_J}[f(y,z) \neq g(y,\bar{0})] \leq 2\eps/3$ and hence $\Pr_{y \sim \U_J,w \sim \U_{\bar{J}} }[f(y,z) \neq f(y,w)]\leq \eps$. In other words, a random restriction of variables outside of $J$ gives, with probability at least $1/2$, a function that is $\eps$-close to $f$. As before we observe that a restriction of a submodular function is a submodular function itself. We therefore can choose $z$ randomly and then run the decision tree representation construction algorithm on $f(y,z)$ as a function of $y$ described in the proof of Theorem \ref{th:submod-lip-rank}. It is easy to see that the running time of the algorithm is essentially determined by the size of the tree. A tree of rank $2k$ over $|J|$ variables has size of at most $|J|^{2k}$ \citep{EhrenfeuchtHaussler:89}. Therefore with probability at least $2/3 \cdot 1/2 = 1/3$, in time $\poly(n, 2^{k^2+k\log{1/\eps}})$ and using $\poly(\log n, 2^{k^2+k\log{1/\eps}})$ value queries we will obtain a submodular function which is $\eps$-close to $f$. As usual the probability of success can be easily boosted to $2/3$ by repeating the algorithm 3 times and testing the hypothesis.
\end{proof}

\section{Proof of Lemma \ref{lem:corr-parity-submod}}
\label{sec:prove-correlation}
Since the functions we are dealing with are going to be symmetric, we make the convenient definition of weight of any $x \in \zo^n$. For any $x \in \zo^n$, the weight of $x$ over a subset $S \subseteq [n]$ of coordinates is defined as $w_S(x) = \sum_{i\in S} x_i $. 

Our correlation bounds for monotone symmetric submodular functions will depend on the following well-known observation which we state without proof.
\begin{fact}[Symmetric Submodular Functions from Concave Profiles]
Let $p:\{0,1,\ldots,n\}: \rightarrow [0,1]$ be any function such that,$$ \forall 0 \leq i \leq n-2\text{, }p(i+1)-p(i) \geq p(i+2) - p(i+1).$$ Let $f_p:\zo^n \rightarrow [0,1]$ be a symmetric function such that $f_p(x) = p(w_{[n]}(x))$. Then $f$ is submodular.
\eat{
Conversely, for any submodular $f:\zo^n \rightarrow [0,1]$, let $p_f:\{0,1,\ldots,n\}: \rightarrow [0,1]$ be the profile of $f$ defined by $\forall i$ $p_f(i) = \frac{1}{{n \choose i}} \sum_{x: w_{[n]}(x) = i} f(x)$. Then, for every $$0 \leq i \leq n-2 \text{, } p_f(i+1)-p_f(i) \geq p_f(i+2) - p_f(i+1).$$
}
\label{concaveprofile}
\end{fact}
\begin{remark}
Observe that for any submodular function $f:\zo^S \rightarrow [0,1]$, the correlation with the parity $\chi_S$ depends only on the profile of $f$,  $p_f:\{0,1,\ldots,n\} \rightarrow [0,1]$, $$\forall i\text{, } p_f(i) = \frac{1}{{n \choose i}} \sum_{x: w_S(x) = i} f(x).$$ That is, if $\tilde{f}:\zo^S \rightarrow [0,1]$ is defined by $\tilde{f}(x) = p_f(w_S(x))$ for every $x \in \zo^n$, then $ \langle f,\chi_S \rangle = \langle \tilde{f}, \chi_S \rangle$. Thus for finding submodular functions with large correlation with a given parity, it is enough to focus on symmetric submodular functions.
\end{remark}

We will need the following well-known formula for the partial sum of binomial coefficients in our correlation bounds.
\begin{fact}[Alternating Binomial Partial Sum]
For every $n,r, k \in \N$,
$$\sum_{j = 0}^r (-1)^j {n \choose j} = (-1)^r {{n-1} \choose r} $$ \label{partialsum}
\end{fact}

\begin{proof}[Proof of Lemma \ref{lem:corr-parity-submod}]
Notice that the parity on any subset $S \subseteq [n]$ of variables at any input $x \in \zo^n$ is computed by $\chi_S(x) = (-1)^{w_S(x)}$.
We will now define a symmetric submodular function $R_S:\zo^S\rightarrow [0,1]$ and then modify it to construct a monotone symmetric submodular function $H_S: \zo^S \rightarrow [0,1]$ that has the required correlation with the associated parity $\chi_S$. It is easy to verify that the natural extension of $R_S$ and $H_S$ to $\zo^n$(from $\zo^S$), that just ignores all the coordinates outside $S$, is submodular and thus it is enough to construct functions on $\zo^S$.

The definition of $R_S$ will vary based on the cardinality of $S$. If $S$ is such that $s = 2k$ for some $k \in \N$, let $R_S$ for each $S \subseteq [n]$ be defined as follows:
\[
R_S(x)=\left\{\begin{array}{cl}
	\frac{w_S(x)}{k}, & w_S(x)\leq k \\
         1 - \frac{w_S(x)-k}{k}, & w_S(x)> k
	   \end{array}\right.
\]

On the other hand, if  $S$ is such that $s = 2k-1$ for some $k \in \N$, define:
\[
R_S(x)=\left\{\begin{array}{cl}
	\frac{w_S(x)}{k-1}, & w_S(x)\leq k-1 \\
         1 - \frac{w_S(x)-k+1}{k-1}, & w_S(x)\geq k
	   \end{array}\right.
\]

Notice that with this definition, $R_S: \zo^n \rightarrow [0,1]$ and has its maximum value exactly equal to $1$. Further, since $R_S$ can be seen to be defined by a concave profile, Fact \ref{concaveprofile} guarantees that $R_S$ is submodular. We will now compute the correlation of $\chi_S$ with $R_S$. We will first deal with the case when $|S|$ is even.

Let $s = 2k$ for some $k\in \N$.

\begin{align*}
\langle R_S, \chi_S \rangle &= \frac{1}{2^{2k}} \sum_{x \in \zo^{2k}} R_S(x) \chi_S(x)\\
&= \frac{1}{2^{2k}}\cdot\sum_{i = 0}^{k}  {{2k} \choose i} (-1)^i  \frac{i}{k} +   \sum_{i = k+1}^{2k}  {{2k} \choose i} (-1)^i  (1 - \frac{i-k}{k})\\
& \text{ Substituting $j = 2k - i$} \\
&= \frac{1}{2^{2k}}\cdot \sum_{i = 0}^{k}  {{2k} \choose i} (-1)^i  \frac{i}{k} + \sum_{j = 0}^{k-1} {{2k} \choose j}(-1)^j \frac{j}{k}\\
&= 2\left(\frac{1}{2^{2k}} \cdot \frac{1}{k} \sum_{i = 0}^{k} {{2k} \choose i} (-1)^i \cdot i\right) - (-1)^k \cdot \frac{1}{2^{2k}} {{2k} \choose k}\\
&= 2\left(\frac{1}{2^{2k}} \cdot \frac{1}{k} \cdot 2k \cdot \sum_{i = 1}^{k} {{2k-1} \choose {i-1}} (-1)^i \right) - (-1)^k \cdot \frac{1}{2^{2k}} {{2k} \choose k}\\
&\text{ Using the partial sum formula from Fact \ref{partialsum} gives:} \\
\langle R_S, \chi_S \rangle &= (-1)^k \cdot \frac{2}{2^{2k}} \cdot \frac{1}{2k-1} {{2k-1} \choose {k}}
\end{align*}
%

Now suppose $s = 2k -1$ for some $k \in \N$.
 \begin{align*}
\langle R_S, \chi_S \rangle &= \frac{1}{2^{2k-1}} \sum_{x \in \zo^{2k-1}} R_S(x) \chi_S(x)\\
&= \frac{1}{2^{2k-1}}\cdot\sum_{i = 0}^{k-1}  {{2k-1} \choose i} (-1)^i  \frac{i}{k-1} +   \sum_{i = k}^{2k-1}  {{2k-1} \choose i} (-1)^i  (1 - \frac{i-k+1}{k-1})\\
& \text{ Substituting $j = 2k - 1-i$} \\
&= \frac{1}{2^{2k-1}}\cdot \sum_{i = 0}^{k}  {{2k} \choose i} (-1)^i  \frac{i}{k} - \sum_{j = 0}^{k-1} {{2k-1} \choose j}(-1)^j \frac{j-1}{k-1}\\
&= \frac{1}{2^{2k-1}} \frac{1}{k-1} \cdot \sum_{j = 0}^{k-1} {{2k-1} \choose j} (-1)^j\\
&\text{ Again, using the partial sum formula from Fact \ref{partialsum} gives:} \\
\langle R_S, \chi_S \rangle &= (-1)^{k+1} \cdot \frac{1}{2^{2k-1}} \cdot \frac{1}{k-1} {{2k-2} \choose {k-1}}
\end{align*}

In either case, we now obtain that $|\langle R_S, \chi_S \rangle| = \Omega(k^{\frac{-3}{2}}) = \Omega(s^{\frac{-3}{2}})$.

For the remaining part of the proof, we need to define the function $H_S$. We obtain $H_S$ by a natural ``monotonization" of $R_S$. Thus, if $s =2k$, let $H_S$ be defined as:
\[
H_S(x)=\left\{\begin{array}{cl}
	\frac{w_S(x)}{k}, & w_S(x)\leq k \\
         1 & w_S(x)> k
	   \end{array}\right.
\]

On the other hand, if  $S$ is such that $s = 2k-1$ for some $k \in \N$, define:
\[
R_S(x)=\left\{\begin{array}{cl}
	\frac{w_S(x)}{k-1}, & w_S(x)\leq k-1 \\
         1 & w_S(x)\geq k
	   \end{array}\right.
\]
Notice again that $H_S:\zo^S \rightarrow [0,1]$ and $H_S$ is submodular by Fact \ref{concaveprofile}.
To obtain a lower bound on $|\langle \chi_S, H_S \rangle|$, $H_S$ can be seen as the average of a monotone linear function and $R_S$, that is, if $s = 2k$, $\forall x$, $H_S(x) = \frac{1}{2} (R_S(x) + \frac{w_S(x)}{k})$ and if $s = 2k-1$, $\forall x$, $H_S(x) = \frac{1}{2} (R_S(x) + \frac{w_S(x)}{k-1})$. It is now easy to obtain a lower bound on the correlation of $\chi_S$ with $H_S$.

\noindent For $s= 2k$, $$\langle \chi_S, H_S \rangle = \frac{1}{2} \langle \chi_S, R_S \rangle + \frac{1}{2} \langle \chi_S, \frac{w_S}{k} \rangle.$$ For $s = 2k-1$,  $$\langle \chi_S, H_S \rangle = \frac{1}{2} \langle \chi_S, R_S \rangle + \frac{1}{2} \langle \chi_S, \frac{w_S}{k-1} \rangle.$$ Finally, observe that for any $s = |S|$, $\langle \chi_S, w_S(x) \rangle = \sum_{i = 0}^{s} {s \choose i} (-1)^i \cdot i = s \sum_{i = 0}^s {{s-1} \choose {i-1}}(-1)^i \cdot i = 0$. This immediately yields the required correlation.

\end{proof}

\end{document}